\documentclass[journal]{IEEEtran}

\usepackage
[
    babel = true, 
]
{microtype}           
\usepackage{csquotes} 
\usepackage{authblk}  

\usepackage[dvipsnames]{xcolor} 

\definecolor{stroke1}{HTML}{2574A9} 

\usepackage{amsmath}
\usepackage{amssymb}
\usepackage{amsthm}
\usepackage{thmtools}
\usepackage{mathtools}
\usepackage{thm-restate}
\usepackage{dsfont}        

\usepackage{url} 
\usepackage{bm}  

\usepackage{graphicx}   
\usepackage{subcaption} 

\usepackage{array}     
\usepackage{booktabs}  
\usepackage{longtable} 
\usepackage{pdflscape} 

\usepackage{enumitem} 

\usepackage
[
    ruled,         
    vlined,        
    linesnumbered, 
]
{algorithm2e} 

\usepackage{xspace}   
\usepackage
[
    shortcuts, 
]
{extdash}             
\usepackage{setspace} 

\usepackage{xparse}    
\usepackage{footnote}  
\usepackage
[
    textsize = scriptsize, 
]
{todonotes}            

\usepackage{balance}
\usepackage{bbold}

\usepackage
[
    bookmarks = true,                 
    bookmarksopen = false,            
    bookmarksnumbered = true,         
    pdfstartpage = 1,                 
    colorlinks = true,                
    allcolors = stroke1,              
]
{hyperref} 

\usepackage
[
    noabbrev,   
    nameinlink, 
]
{cleveref} 
\makeatletter
    \def\IfEmptyTF#1%
    {%
        \if\relax\detokenize{#1}\relax%
            \expandafter\@firstoftwo%
        \else%
            \expandafter\@secondoftwo%
        \fi%
    }
\makeatother

\NewDocumentCommand{\mathOrText}{m}
{%
    \ensuremath{#1}\xspace%
}

\let\originalleft\left
\let\originalright\right
\renewcommand{\left}{\mathopen{}\mathclose\bgroup\originalleft}
\renewcommand{\right}{\aftergroup\egroup\originalright}

\makeatletter
    \DeclareRobustCommand{\bfseries}%
    {%
        \not@math@alphabet\bfseries\mathbf%
        \fontseries\bfdefault\selectfont%
        \boldmath%
    }

\xspaceaddexceptions{]\}}

\allowdisplaybreaks

\crefname{ineq}{inequality}{inequalities}
\creflabelformat{ineq}{#2{\upshape(#1)}#3}

\crefname{term}{term}{terms}
\creflabelformat{term}{#2{\upshape(#1)}#3}

\crefname{cond}{condition}{conditions}
\creflabelformat{term}{#2{\upshape(#1)}#3}

\makesavenoteenv{figure}
\makesavenoteenv{table}
\makesavenoteenv{tabular}

\newtheorem{theorem}{Theorem}
\newtheorem{lemma}[theorem]{Lemma}
\newtheorem{corollary}[theorem]{Corollary}

\NewDocumentCommand{\functionTemplate}{m m m m o}%
{%
    \IfNoValueTF{#5}%
    {%
        \mathOrText{#1\left#2{#4}\right#3}%
    }%
    {%
        \mathOrText{#1#5#2{#4}#5#3}%
    }%
}

\newcommand*{\leftBracketType}{(}
\newcommand*{\rightBracketType}{)}

\NewDocumentCommand{\createFunction}{m m o o}%
{%
    \renewcommand*{\leftBracketType}{\IfNoValueTF{#3}{(}{#3}}%
    \renewcommand*{\rightBracketType}{\IfNoValueTF{#4}{)}{#4}}%
    \NewDocumentCommand{#1}{o o}%
    {%
        \IfNoValueTF{##1}%
        {%
            \mathOrText{#2}%
        }%
        {%
            \functionTemplate{#2}{\leftBracketType}{\rightBracketType}{##1}[##2]%
        }%
    }%
}

\DeclareDocumentCommand{\probabilisticFunctionTemplate}{m m O{} o}
{%
    \functionTemplate{#1}%
    {\lbrack}%
    {\rbrack}%
    {#2\IfEmptyTF{#3}{}{\ \IfNoValueTF{#4}{\left}{#4}\vert\ \vphantom{#2}#3\IfNoValueTF{#4}{\right.}{}}}%
    [#4]%
}

\newcommand*{\N}{\mathOrText{\mathds{N}}}

\newcommand*{\R}{\mathOrText{\mathds{R}}}

\newcommand*{\indicatorFunctionSymbol}{\mathds{1}}

\RenewDocumentCommand{\Pr}{m O{} o}%
{%
    \probabilisticFunctionTemplate{\mathrm{Pr}}{#1}[#2][#3]%
}

\NewDocumentCommand{\E}{m O{} o}%
{%
    \probabilisticFunctionTemplate{\mathrm{E}}{#1}[#2][#3]%
}

\NewDocumentCommand{\Var}{m O{} o}%
{%
    \probabilisticFunctionTemplate{\mathrm{Var}}{#1}[#2][#3]%
}

\DeclareDocumentCommand{\bigO}{m o}%
{%
    \functionTemplate{\mathrm{O}}{(}{)}{#1}[#2]%
}

\DeclareDocumentCommand{\smallO}{m o}%
{%
    \functionTemplate{\mathrm{o}}{(}{)}{#1}[#2]%
}

\DeclareDocumentCommand{\bigTheta}{m o}%
{%
    \functionTemplate{\Theta}{(}{)}{#1}[#2]%
}

\DeclareDocumentCommand{\bigOmega}{m o}%
{%
    \functionTemplate{\Omega}{(}{)}{#1}[#2]%
}

\DeclareDocumentCommand{\smallOmega}{m o}%
{%
    \functionTemplate{\omega}{(}{)}{#1}[#2]%
}

\DeclareDocumentCommand{\eulerE}{o}%
{%
    \mathOrText{\mathrm{e}\IfNoValueTF{#1}{}{^{#1}}}%
}

\DeclareDocumentCommand{\poly}{m o}%
{%
    \functionTemplate{\mathrm{poly}}{(}{)}{#1}[#2]%
}

\createFunction{\id}{\mathrm{id}}

\NewDocumentCommand{\ind}{m o o}%
{%
    \IfNoValueTF{#2}%
    {%
        \mathOrText{\indicatorFunctionSymbol_{#1}}%
    }%
    {%
        \functionTemplate{\indicatorFunctionSymbol_{#1}}{(}{)}{#2}[#3]%
    }%
}

\DeclareDocumentCommand{\dom}{m o}%
{%
    \functionTemplate{\mathrm{dom}}{(}{)}{#1}[#2]%
}

\DeclareDocumentCommand{\rng}{m o}%
{%
    \functionTemplate{\mathrm{rng}}{(}{)}{#1}[#2]%
}

\DeclareDocumentCommand{\d}{o}%
{%
    \mathrm{d}\IfNoValueTF{#1}{}{^{#1}}%
}

\DeclareDocumentCommand{\set}{m m o}%
{%
    \mathOrText{\IfNoValueTF{#3}{\left}{#3}\{#1\ \IfNoValueTF{#3}{\left}{#3}\vert\ \vphantom{#1}#2\IfNoValueTF{#3}{\right.}{}\IfNoValueTF{#3}{\right}{#3}\}}%
}

\newcommand\numberthis{\addtocounter{equation}{1}\tag{\theequation}}

\newcommand*{\indicator}[1]{\mathOrText{\mathds{1}{(#1)}}}

\DeclareDocumentCommand{\randomProcess}{o}
{%
    \mathOrText{X\IfNoValueTF{#1}{}{_{#1}}}%
}

\DeclareDocumentCommand{\transformedProcess}{o}
{%
    \mathOrText{Y\IfNoValueTF{#1}{}{_{#1}}}%
}

\DeclareDocumentCommand{\transformedProcessOther}{o}
{%
    \mathOrText{Z\IfNoValueTF{#1}{}{_{#1}}}%
}

\DeclareDocumentCommand{\randomProcessOther}{o}
{%
    \mathOrText{Y\IfNoValueTF{#1}{}{_{#1}}}%
}

\DeclareDocumentCommand{\randomProcessThird}{o}
{%
    \mathOrText{Z\IfNoValueTF{#1}{}{_{#1}}}%
}

\DeclareDocumentCommand{\filtration}{o}
{%
    \mathOrText{\mathcal{F}\IfNoValueTF{#1}{}{_{#1}}}%
}

\DeclareDocumentCommand{\filtrationOther}{o}
{%
    \mathOrText{\mathcal{G}\IfNoValueTF{#1}{}{_{#1}}}%
}

\DeclareDocumentCommand{\filtrationSecondOther}{o}
{%
    \mathOrText{\mathcal{G}'\IfNoValueTF{#1}{}{_{#1}}}%
}

\DeclareDocumentCommand{\individual}{o}
{%
    \mathOrText{\boldsymbol{x}\IfNoValueTF{#1}{}{^{(#1)}}}%
}

\DeclareDocumentCommand{\individualOther}{o}
{%
    \mathOrText{\boldsymbol{y}\IfNoValueTF{#1}{}{^{(#1)}}}%
}

\DeclareDocumentCommand{\individualThird}{o}
{%
    \mathOrText{\boldsymbol{z}\IfNoValueTF{#1}{}{^{(#1)}}}%
}

\DeclareDocumentCommand{\difference}{o}
{%
    \mathOrText{\Delta\IfNoValueTF{#1}{}{_{#1}}}%
}

\DeclareDocumentCommand{\runTimePoints}{m}
{%
    \mathOrText{N_{\textrm{#1}}}%
}

\DeclareDocumentCommand{\returnTimePoints}{m}
{%
    \mathOrText{M_{\textrm{#1}}}%
}

\DeclareDocumentCommand{\returnTimePointsStart}{o}
{%
    \mathOrText{R_{\IfNoValueTF{#1}{}{#1}}}%
}

\DeclareDocumentCommand{\returnTimePointsEnd}{o}
{%
    \mathOrText{O_{\IfNoValueTF{#1}{}{#1}}}%
}

\DeclareDocumentCommand{\retry}{o}
{%
    \mathOrText{M\IfNoValueTF{#1}{}{_{#1}}}%
}

\DeclareDocumentCommand{\sequence}{m o o o}%
{%
    \mathOrText{\IfNoValueTF{#2}{#1}{\IfEmptyTF{#2}{#1}{\IfNoValueTF{#4}{\left}{#4}(#1_{#2}\IfNoValueTF{#4}{\right}{#4})_{#2 \in #3}}}}%
}

\newcommand*{\timePoint}{\mathOrText{t}}

\newcommand*{\stoppingTime}{\mathOrText{T}}
\newcommand*{\stoppingTimeOther}{\mathOrText{S}}
\newcommand*{\stoppingTimeThird}{\mathOrText{U}}
\newcommand*{\drift}{\mathOrText{\delta}}
\newcommand*{\defeq}{\mathOrText{\coloneqq}}
\newcommand*{\eqdef}{\mathOrText{\eqqcolon}}
\newcommand*{\plateauDistance}{\mathOrText{r}}
\newcommand*{\basisRLS}{\mathOrText{\lambda}}
\newcommand*{\potentialFunction}{\mathOrText{g}}

\newcommand*{\runTimeSet}{\mathOrText{\mathcal{R}}}

\newcommand*{\myAlgorithm}{\mathOrText{A}}
\newcommand*{\probabilityOfSuccessAfterFailure}{\mathOrText{p}}
\newcommand*{\probabilityOfImmediateSuccess}{\mathOrText{\probabilityOfSuccessAfterFailure_0}}
\newcommand*{\myIndex}{\mathOrText{i}}
\newcommand*{\majority}{\mathOrText{\textsc{Majority}^{(1)}_{0}}}
\newcommand*{\basisRLSNumerator}{\mathOrText{a}}
\newcommand*{\onemax}{\textsc{OneMax}\xspace}
\newcommand*{\jump}{\textsc{Jump}\xspace}
\newcommand*{\initialization}{\mathOrText{D}}
\newcommand*{\RExtended}{\mathOrText{\overline{\R}}}
\newcommand*{\eventForRestart}{\mathOrText{L}}

\DeclareDocumentCommand{\plateau}{o}
{%
    \mathOrText{\textsc{HasMajority}\IfNoValueTF{#1}{}{_{#1}}}%
}
\DeclareDocumentCommand{\asymmetricPlateau}{o}
{%
    \mathOrText{\textsc{Majority}\IfNoValueTF{#1}{}{_{#1}}}%
}
\DeclareDocumentCommand{\rls}{o}
{%
    RLS\IfNoValueF{#1}{$_{#1}$}\xspace%
}
\DeclareDocumentCommand{\onePlusOneEA}{o}
{%
    \mathOrText{(1 + 1)~\textrm{EA}\IfNoValueF{#1}{_{#1}}}%
}

\DeclareDocumentCommand{\mutation}{o}
{%
    \mathOrText{\phi\IfNoValueTF{#1}{}{_{#1}}}%
}

\crefname{case}{case}{cases}
\creflabelformat{case}{#2{\upshape#1}#3}

\title{Run Time Analysis for Random Local Search\\ on Generalized Majority Functions}

\author{%
    Carola~Doerr and %
    \thanks{C.~Doerr is with Sorbonne Universit\'e, CNRS, LIP6, Paris, France}%
    Martin~S.~Krejca%
    \thanks{Martin~S.~Krejca is with LIX, CNRS, Ecole Polytechnique, Institut Polytechnique de Paris, Palaiseau, France.}%
    \thanks{This work was financially supported by the Paris Île-de-France Region via the DIM RFSI AlgoSelect project and via the European Union's Horizon 2020 research and innovation program under the Marie Skłodowska-Curie grant agreement No. 945298-ParisRegionFP.}%
}

\begin{document}

\maketitle

\begin{abstract}
    Run time analysis of evolutionary algorithms recently makes significant progress in linking algorithm performance to \emph{algorithm} parameters.
    However, settings that study the impact of \emph{problem} parameters are rare.
    The recently proposed W-model provides a good framework for such analyses, generating pseudo-Boolean optimization problems with \emph{tunable} properties.

	We initiate theoretical research of the W-model by studying how one of its properties---neutrality---influences the run time of random local search.
    Neutrality creates plateaus in the search space by first performing a majority vote for subsets of the solution candidate and then evaluating the smaller-dimensional string via a low-level fitness function.

    We prove upper bounds for the expected run time of random local search on this \asymmetricPlateau problem for its entire parameter spectrum. To this end, we provide a theorem, applicable to many optimization algorithms, that links the run time of \asymmetricPlateau with its symmetric version \plateau, where a sufficient majority is needed to optimize the subset. We also introduce a generalized version of classic drift theorems as well as a generalized version of Wald's equation, both of which we believe to be of independent interest.
\end{abstract}


\section{Introduction}
\label{sec:introduction}
Randomized search heuristics, such as evolutionary algorithms (EAs), have been applied with great success to real-world optimization problems for which the user is faced with limited information about the problem or limited resources to solve it.
This is an impressive feat, since such problems are typically hard due to a combination of different challenging features.
In order to better understand the reasons behind why EAs perform so well in such settings, theoretical results analyze benchmark functions that encapsulate specific features assumed to be also present in real-world problems.
At present, these functions are typically unrelated to each other, and it is difficult to derive how they interact with one another.
This makes it hard to translate such theoretical results to more realistic problems.
An alternative approach for deriving theoretical results is to consider a class of different functions that can be combined in well-defined ways, allowing to provide guarantees for more complex problems.
The recently introduced W-model~\cite{WeiseCLW20WModel} provides such function classes.
Given a pseudo-Boolean function~$f$, the W-model proposes a sequence of four independent steps that each change a different feature of~$f$.
The degree by which each feature is changed is regulated by parameters specific to each step.
Thus, overall, the W-model allows to create a diverse set of functions, all based on~$f$, that showcases various features in different intensities.

\begin{figure}
    \includegraphics[width = \columnwidth]{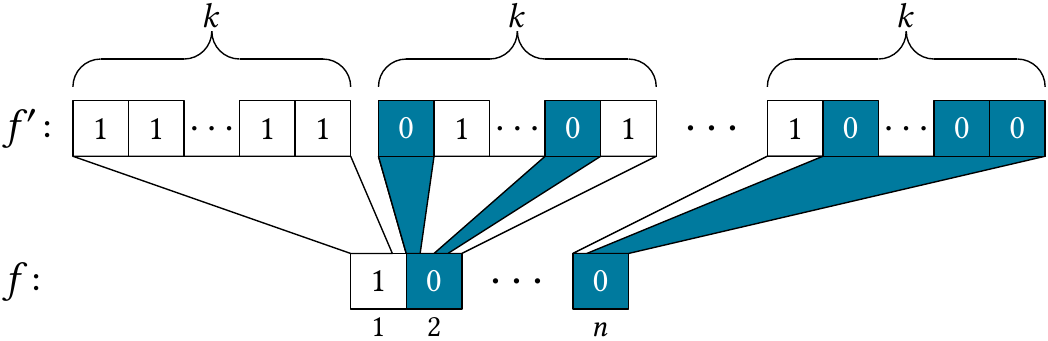}
    \caption{\label{fig:neutrality}
        Given a pseudo-Boolean function~$f$ defined over bit strings of length $n \in \N_{> 0}$ and a parameter $k \in \N_{> 0}$, \emph{neutrality} creates a function~$f'$ defined over bit strings of length~$nk$, consisting of~$n$ blocks of length~$k$ each.
        When evaluating~$f'$, first, for each of the~$k$ blocks, the majority of its bits is determined.
        Then,~$f$ is evaluated on the resulting bit string of length~$n$ of the majority outcomes.
    }
\end{figure}

In this article, we initiate the run time analysis of the W-model.
To this end, we focus on one of its features called \emph{neutrality}.
Given a function~$f$ defined over bit strings of length~$n$ as well as an integer parameter~$k$, neutrality substitutes each bit of a length-$n$ bit string by~$k$ bits.
When evaluating a new bit string of length~$nk$, each of the subsequent blocks of~$k$ bits is reduced to the majority among its bit values, and the resulting bit string of length~$n$ is evaluated by~$f$ (see \Cref{fig:neutrality}).
Overall, neutrality introduces a two-stage process into the optimization of~$f$.
On a high level, $f$ is optimized.
On a low level, the correct majority value of each block is determined.
Assuming that the different blocks are independent, analyzing the run times of the two stages separately and then combining these results leads to a run time result for~$f$.

Doerr et~al.~\cite{DoerrSW13BlockAnalysis} analyzed the high-level approach of the aforementioned two-stage process diligently for the class of \emph{separable} functions, where a function is separable if its value is the sum of the values of~$n$ independent sub-functions.
The authors show that the run time of random local search (\rls) and of the \onePlusOneEA is, in expectation, bounded from above by the slowest expected run time of a sub-function times $\ln(n)$.
The logarithmic overhead is a result of the classical coupon collector problem that accounts for optimizing \emph{all}~$n$ sub-functions correctly.

The second stage of the two-stage optimization process sketched above concerns analyzing the run time of a majority function.
Since the value returned by such a function only changes if the majority of the bits in its input changes, this leads to analyzing the behavior of an algorithm in a landscape of equal fitness, called a \emph{plateau}.
Changing the number of bits to the optimal majority requires an EA to perform a random walk on the plateau until finding the optimal majority, at which point in time the block is optimized.
We refer to this search behavior as \emph{crossing} the plateau.

In order to better understand the impact of the plateau size on the optimization process, we extend the notion of majority for the W-model.
To this end, we introduce the function \plateau[\plateauDistance], defined over bit strings of length $n \in \N^+$ (where~$n$ is even\footnote{Note that, in contrast to \Cref{fig:neutrality}, we denote the dimension of \plateau[\plateauDistance] (defined over a block) now by~$n$ and not anymore by~$k$.}), having a parameter $\plateauDistance \in [0, n/2] \cap \N$ how strong the majority has to be in order to be counted.
\plateau[\plateauDistance] takes the two values~$0$ and~$1$.
A bit string has a \plateau[\plateauDistance]-value of~$1$ if its number of~$0$s or its number of~$1$s is at least $n/2 + \plateauDistance$, and it has a value of~$0$ otherwise.
We also introduce the asymmetric version \asymmetricPlateau[\plateauDistance], where a bit string only has a function value of~$1$ if it has at least $n/2 + \plateauDistance$ $1$s (and it is~$0$ otherwise).

\paragraph{Our results}
We analyze \rls, which iteratively constructs new solutions by flipping exactly one bit, chosen uniformly at random, in its currently best solution.
We bound the expected run time of \rls, that is, the number of iterations until an optimal solution is found for the first time, on \plateau and \asymmetricPlateau from above for the entire range of~\plateauDistance.
Using the results by Doerr et~al.~\cite{DoerrSW13BlockAnalysis}, this implies run time bounds for \onemax with added neutrality (\Cref{cor:rls_w-model}).
We first analyze the symmetric case of \plateau (similar to the result by Bian et~al.~\cite{BianQTY20RobustLinearOptimization}) and show with \Cref{thm:eas_on_plateaus} the complex dependency of the run time of \rls on the parameter~\plateauDistance, resulting in multiple different regimes (\Cref{fig:runtime_dependency}).
For values of~\plateauDistance constant in~$n$, the run time is constant, and it grows quadratically in~\plateauDistance until $r = \bigO{\sqrt{n}}$.
For $r = \smallOmega{\sqrt{n}} \cap \smallO{n}$, the expected run time is mainly given by an exponential in $\plateauDistance^2/n$.
For the remaining regime of~\plateauDistance, the expected run time is at least exponential.
The behavior for \asymmetricPlateau is similar, as \Cref{thm:rls_run_time_on_asymmetric_plateau} shows.
The main difference is an overhead of $n \ln(\plateauDistance)$, which is caused by the larger size of the plateau.
For the special case $r = n/2 - 1$, the \asymmetricPlateau function reduces to the well-known \textsc{Needle} function, for which the run time of RLS, the \onePlusOneEA (EA), and a number of generalizations are very well understood~\cite{GarnierKS99,AntipovD18PrecisePlateauAnalysis}.

Our result for \asymmetricPlateau is based on a restart argument applied to \plateau.
Since such restart arguments are not specific to the context of EAs or even optimization, we analyze the situation in a general setting (\Cref{thm:general_restart_argument}).
Our result applies to any assortment of interleaved stopping times.
We prove \Cref{thm:general_restart_argument} via a generalized version of Wald's equation (\Cref{thm:walds_equation}), which we could not find in this detail and generality in the literature.
We prove the generalized version of Wald's equation via a new general basic drift theorem (\Cref{thm:basic_drift}), which relates the expected progress of a random process more explicitly to the values at the start and at the end of the process than traditional drift theorems.
As all of these theorems relate to scenarios commonly found in the analysis of EAs, we are confident that they are of independent interest.

With \Cref{thm:restart_argument}, we explicitly connect the expected run times of \plateau and \asymmetricPlateau for a large class of EAs, providing a general framework for our setting.
This result comes with multiple conditions, due to its generality.
However, we also provide additional statements (\Cref{lem:number_of_restarts,lem:integrability_of_the_stopping_time}) that show that many EAs satisfy the conditions of \Cref{thm:restart_argument}.
In order to apply \Cref{thm:restart_argument} to any EA we cover with this framework, it then remains to provide bounds on the expected run time of an EA on \plateau.
We note that this framework is also useful in the setting of deletion-robust optimization, as studied by Bian et~al.~\cite{BianQTY20RobustLinearOptimization}, since \asymmetricPlateau occurs as a subproblem there.

We complement our theoretical analyses with empirical investigations of the \rls variant \rls[\ell], which flips exactly $\ell \in [1, n] \cap \N$ bits (chosen uniformly among all $\ell$-cardinality subsets), on \asymmetricPlateau.
For $r = \lfloor \sqrt{n} \rfloor$, we consider the empirical run time for different values of~$\ell$, and we find that a value around~$n/2$ leads to the lowest run time (\Cref{fig:empirical_run_time_varying_n_and_k}), which is by orders of magnitude lower than the empirical run time of conventional \rls.
Additional results, which look into the trajectory of the best solution so far over time, reveal that while instantiations of \rls[\ell] for different values of~$\ell$ have qualitatively the same random-walk behavior, the larger step size for larger values of~$\ell$ makes it easier to find a global optimum (\Cref{fig:different_trajectories}).
However, we note that there is a limit to how many bits to flip during a single iteration makes sense, as, for example, inverting the bit string in each iteration makes \rls[n] alternate between just two solutions.
Nonetheless, a high mutation rate in the order of~$n$ seems beneficial for our setting.

\begin{figure}
    \centering
    \includegraphics[width = 0.86\columnwidth]{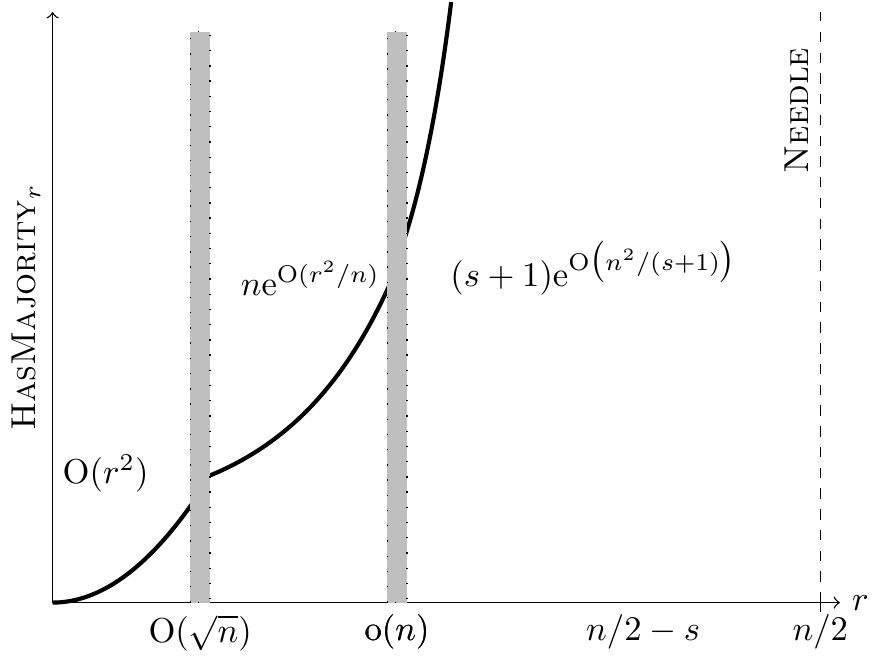}
    \caption{\label{fig:runtime_dependency}
        The dependency of the expected run time of \rls on \plateau[\plateauDistance] with respect to the majority parameter~\plateauDistance as proven in \Cref{thm:eas_on_plateaus}.
        The gray boxes denote regions in which the asymptotics switch (sharply) from \bigO{\sqrt{n}} to \smallOmega{\sqrt{n}} and from \smallO{n} to \bigTheta{n}.
        The value $s \in \N$ is any value such that there is a $c \in \bigTheta{1}$ with $c < 1/2$ such that $s \leq cn$.
        The dashed line at $n/2$ denotes the run time on \textsc{Needle}.
    }
\end{figure}

\paragraph{Related work}
Analyzing the performance of EAs on plateaus by theoretical means is a well-established concept.
The arguably most studied benchmark function with a plateau is \jump, existing in different variants, many of which were proposed recently~\cite{AntipovD18PrecisePlateauAnalysis,BamburyBD21GeneralizedJump,AntipovN21RealJump,Witt21EDAsGeneralizedJump}.
However, the problem posed by the plateau in \jump is typically different than from crossing a plateau, as we discuss in the following.
In its classical variant, \jump consists of a slope toward the unique global optimum, which itself is surrounded by solutions with worse function values (called the \emph{valley}).
The plateau is the set of points with second-best function value.
For \jump, the behavior of an EA on the plateau is typically not that important, but rather that the valley is overcome, oftentimes referred to as \emph{jumping} over the valley.
This is in contrast to crossing a plateau, where the behavior of an EA \emph{on} the plateau is very important.
The main importance of jumping remains true for \jump when shifting the global optimum~\cite{AntipovN21RealJump} and/or increasing the number of global optima to jump to~\cite{BamburyBD21GeneralizedJump,Witt21EDAsGeneralizedJump}.
However, the picture changes when considering EAs that apply crossover, that is, combining different solutions when creating new ones.
Then, the random walk on the plateau becomes relevant for the expected run time on \jump~\cite{DangFKKLOSS16DiversityMechanisms,DangFKKLOSS18JumpCrossover,WhitleyVHM18MajorityJump}.
Nonetheless, in such a setting, it is primarily important to create diverse solutions on the plateau, in order to leave it via crossover, instead of crossing the plateau (by finding a better solution next to it).

One variant of \jump, recently introduced by Antipov and Doerr~\cite{AntipovD18PrecisePlateauAnalysis}, which removes the valley and increases the plateau, actually requires dealing with crossing the plateau.
The authors analyze the \onePlusOneEA on this variant and determine an exact run time bound (up to lower-order terms) via arguments on Markov chains.
In contrast to our setting, where the set of global optima is typically very large (in the order of $\binom{n}{n/2 + \plateauDistance}$), their function has a unique global optimum.

In another recent paper, Bian et~al.~\cite{BianQTY20RobustLinearOptimization} analyzed the \onePlusOneEA on \asymmetricPlateau as a side result when considering deletion-robust linear optimization.
Using our notion of \asymmetricPlateau, their result holds for values of $\plateauDistance = \bigO{\sqrt{n \log n}}$, that is, not for the entire range of~\plateauDistance up to~$n/2$.
Further, their result does not seem to be tight, as for values of $r = \bigTheta{1}$, their bound is \bigO{n^2}, whereas we show a constant bound, albeit for \rls.

\paragraph{Outline}
In \Cref{sec:preliminaries}, we introduce our notation, formalize our problem, and provide the mathematical tools we use for our analysis.
Especially, we prove \Cref{thm:basic_drift,thm:walds_equation}, which we believe are of independent interest.
In \Cref{sec:eas_on_plateaus}, we analyze \rls on \plateau (\Cref{thm:eas_on_plateaus}) and \asymmetricPlateau (\Cref{thm:rls_run_time_on_asymmetric_plateau}).
Further, we provide the theorem for decomposing stopping times (\Cref{thm:general_restart_argument}), show how it can be used to relate the expected run times of EAs on \plateau and \asymmetricPlateau (\Cref{thm:restart_argument}), and put everything together to prove our main result, the run time bound for \rls optimizing \onemax with added neutrality from the W-model (\Cref{cor:rls_w-model}).
\Cref{sec:experiments} describes our empirical results.
Last, we conclude with \Cref{sec:conclusion} and suggest ideas for future work.

\section{Preliminaries}
\label{sec:preliminaries}
Let~\N denote the set of natural numbers, including~$0$.
For all $m, n \in \N$, let $[m .. n] \defeq [m, n] \cap \N$, and let $[m] \defeq [1 .. m]$.
Further, let~\R denote the set of reals, and let $\RExtended = \R \cup \{-\infty, \infty\}$.
For the sake of conciser notations, we define that $\infty \cdot  0 = 0 = -\infty \cdot 0$.
We say that a random variable~\randomProcess over~\R is \emph{integrable} if and only if $\E{|\randomProcess|} < \infty$.
Note that integrability of~\randomProcess implies that~\randomProcess is almost surely finite.
We extend this concept to random processes and say that $(\randomProcess[\timePoint])_{\timePoint \in \N}$ is integrable if and only if for each $\timePoint \in \N$, it holds that~\randomProcess[\timePoint] is integrable.

Let $n \in \N_{> 0}$.
We call each $\individual \in \{0, 1\}^n$ an \emph{individual} and its components \emph{bits.}
Further, for all $b \in \{0, 1\}$, let $|\individual|_b \defeq |\{i \in [n] \mid \individual_i = b\}|$.
We call any real-valued function over individuals a \emph{fitness function.}

Let $n \in \N_{> 0}$ be even, and let $\plateauDistance \in [0 .. n / 2]$.
We consider the two fitness functions \plateau[\plateauDistance] and \asymmetricPlateau[\plateauDistance], where, for all $\individual \in \{0, 1\}^n$, it holds that
\begin{align*}
    \plateau[\plateauDistance](\individual) &= \indicator{\max\{|\individual|_0, |\individual|_1\} \geq n/2 + r} \textrm{ and}\\
    \asymmetricPlateau[\plateauDistance](\individual) &= \indicator{|\individual|_1 \geq n/2 + r} ,
\end{align*}
where $\indicator{\mathcal{E}}$ is the indicator function that evaluates to~$1$ if event $\mathcal{E}$ occurs and that evaluates to~$0$ otherwise.
When we drop the index~\plateauDistance and only mention \plateau or \asymmetricPlateau, we mean that the respective statement holds for all $\plateauDistance \in [0 .. n / 2]$.

We note that \plateau and \asymmetricPlateau are special cases of the functions mentioned in \Cref{sec:introduction} with the target string $\individualThird \in \{0, 1\}^n$ being the all-$1$s string.
Since we only consider algorithms that treat~$0$s and~$1$s symmetrically, analyzing \plateau and \asymmetricPlateau as defined above covers without loss of generality the more general function class discussed in \Cref{sec:introduction}.

We consider \emph{evolutionary algorithms} (EAs) to be (possibly random) sequences $(\individual[\timePoint])_{\timePoint \in \N}$ over $\{0, 1\}^n$.
Typically, the sequence of an EA~\myAlgorithm depends (among others) on a fitness function~$f$, and is not well-defined otherwise.
We refer to well-defined sequences by saying that \emph{\myAlgorithm optimizes~$f$.}

A specific EA that we are interested in and that optimizes a fitness function is \emph{random local search} (RLS, \Cref{alg:rls}).
Note that for RLS to optimize~$f$ and the respective sequence to be well-defined, in addition to~$f$, an initialization distribution needs to be specified.
Typically, the uniform distribution is chosen.
However, since we require for some of our results that RLS starts at a specific position in the search space, we consider general initialization distributions.

Given an EA $(\individual[\timePoint])_{\timePoint \in \N}$ optimizing a fitness function~$f$, we call $\inf \{\timePoint \in \N \mid f\big(\individual[\timePoint]\big) = \max \rng{f}\}$ the \emph{run time} of the algorithm on~$f$.
Note that the set that the infimum is taken over may be empty.
To this end, we define $\inf \emptyset = \infty$.
Further note that the run time can be a random variable.
We refer to the expectation of the run time as the \emph{expected run time.}

\begin{algorithm}[t]
    \caption{
        \label{alg:rls}
        Random local search (\rls) with initialization distribution~\initialization, maximizing $f\colon \{0, 1\}^n \to \R$.
    }
    $\individual[0] \sim \initialization$\;
    \For{$\timePoint \in \N$}
    {
        $\individualOther \gets$ flip exactly one bit in~\individual[\timePoint], chosen uniformly at random\;
        \lIf{$f(\individualOther) \geq f\big(\individual[\timePoint]\big)$}
        {%
            $\individual[\timePoint + 1] \gets \individualOther$%
        }
        \lElse
        {%
            $\individual[\timePoint + 1] \gets \individual[\timePoint]$%
        }
    }
\end{algorithm}

\subsection{Tools}

Besides basic tools from probability theory, we apply drift analysis~\cite{Lengler20DriftBookChapter}.
Most prominently, we use the additive-drift theorem, originally introduced by He and Yao~\cite{HeY01DriftTheory}, as well as the multiplicative-drift theorem, originally introduced by Doerr et~al.~\cite{DoerrJW12MultiplicativeDrift}.
We state the theorems in a general fashion, as provided by Kötzing and Krejca~\cite{KoetzingK19DriftTheory}, removing some unnecessary assumptions.

\begin{theorem}[Additive drift, upper bound {\cite[Theorem~$7$]{KoetzingK19DriftTheory}}]
    \label{thm:additive_drift}
    Let $(\randomProcess[\timePoint])_{\timePoint \in \N}$ be an integrable random processes over $\R_{\geq 0}$, adapted to a filtration $(\filtration[\timePoint])_{\timePoint \in \N}$, and let~\stoppingTime be a stopping time with respect to~\filtration.
    Assume that there is a $\drift \in \R_{> 0}$ such that, for all $\timePoint \in \N$, it holds that
    \begin{align}
        \label[ineq]{eq:additive-drift:drift_condition}
        \big(\randomProcess[\timePoint] - \E{\randomProcess[\timePoint + 1]}[\filtration[\timePoint]]\big) \cdot \indicator{\timePoint < \stoppingTime} \geq \drift \cdot \indicator{\timePoint < \stoppingTime} .
    \end{align}
    Then $\E{\stoppingTime}[\filtration[0]] \leq (\randomProcess[0] - \E{\randomProcess[\stoppingTime]}[\filtration[0]]) / \drift$.
\end{theorem}

\begin{theorem}[Multiplicative drift, upper bound {\cite[Corollary~$16$]{KoetzingK19DriftTheory}}]
    \label{thm:multiplicative_drift}
    Let $(\randomProcess[\timePoint])_{\timePoint \in \N}$ be an integrable random processes over $\R$, adapted to a filtration $(\filtration[\timePoint])_{\timePoint \in \N}$, with $\randomProcess[0] \geq 1$, and let $\stoppingTime = \inf\{\timePoint \in \N \mid \randomProcess[\timePoint] < 1\}$.
    Assume that there is a $\drift \in \R_{> 0}$ such that, for all $\timePoint \in \N$, it holds that
    \begin{align}
        \label[ineq]{eq:multiplicative-drift:drift_condition}
        \big(\randomProcess[\timePoint] - \E{\randomProcess[\timePoint + 1]}[\filtration[\timePoint]]\big) \cdot \indicator{\timePoint < \stoppingTime} \geq \drift \cdot \randomProcess[\timePoint] \cdot \indicator{\timePoint < \stoppingTime} .
    \end{align}
    Then $\E{\stoppingTime}[\filtration[0]] \leq \big(1 + \ln(\randomProcess[0])\big) / \drift$.
\end{theorem}

In addition to these theorems, we apply a general version of Wald’s equation (\Cref{thm:walds_equation}), which we prove using a new general drift theorem (\Cref{thm:basic_drift}), as we could not find a statement this general in the literature.
We prove \Cref{thm:basic_drift} by using the optional-stopping theorem (\Cref{thm:optional_stopping}), similar to how Kötzing and Krejca~\cite{KoetzingK19DriftTheory} derived their results.

\begin{theorem}[Optional stopping {\cite[Theorems~$4.8.4$ and~$4.8.5$]{Durrett19ProbabilityTheoryExamples}, \cite[Slide~$12$]{Kovchegov2013}}]
    \label{thm:optional_stopping}
    Let~\stoppingTimeOther be a stopping time with respect to a filtration \sequence{\filtration}[\myIndex][\N], and let \sequence{\randomProcess}[\myIndex][\N] be a random process over~\RExtended, adapted to~\filtration.
    \begin{enumerate}[label = \alph*)]
        \item \label[case]{label:optional_stopping:supermartingale}
        If~\randomProcess is a non-negative supermartingale, then $\E{\randomProcess[\stoppingTimeOther]}[\filtration[0]] \leq \randomProcess[0]$.

        \item \label[case]{label:optional_stopping:submartingale}
        If $\E{\stoppingTimeOther} < \infty$ and if~\randomProcess is a submartingale such that there is a $c \in \R$ such that, for all $\myIndex \in \N$, it holds that $\E{|\randomProcess[\myIndex] - \randomProcess[\myIndex + 1]|}[\filtration[\myIndex]] \cdot \indicator{\myIndex < \stoppingTimeOther} \leq c$, then $\E{\randomProcess[\stoppingTimeOther]}[\filtration[0]] \geq \randomProcess[0]$.
        \qedhere
    \end{enumerate}
\end{theorem}

The basic-drift theorem below is a generalization of the most commonly used drift theorems.
However, in contrast to those, it relates the accumulated drift to the overall change in potential instead of explicitly solving for the first-hitting time.
It is thus more akin to the (deterministic) potential method~\cite[Chapter~$17.3$]{CormenLRS09IntroductionToAlgorithms}, extending it to random variables.
This allows to estimate the accumulated \emph{cost} of operations rather than only the \emph{amount} of operations.

\begin{theorem}[Basic drift]
    \label{thm:basic_drift}
    Let~\stoppingTimeOther be a stopping time with respect to a filtration \sequence{\filtration}[\myIndex][\N], and let \sequence{\randomProcess}[\myIndex][\N] and \sequence{\drift}[\myIndex][\N] be random processes over~\RExtended such that $(\randomProcess[\myIndex] \cdot \indicator{\myIndex \leq \stoppingTimeOther})_{\myIndex \in \N}$ and $(\drift_\myIndex \cdot \indicator{\myIndex < \stoppingTimeOther})_{\myIndex \in \N}$ are integrable and adapted to~\filtration\!\!.
    \begin{enumerate}[label = \alph*)]
        \item \label[case]{label:basic_drift:upper_bound}
        Assume that $(\randomProcess[\myIndex] \cdot \indicator{\myIndex \leq \stoppingTimeOther})_{\myIndex \in \N}$ and $(\drift_\myIndex \cdot \indicator{\myIndex < \stoppingTimeOther})_{\myIndex \in \N}$ are non-negative and that for all $\myIndex \in \N$, it holds that
        \begin{equation}
            \tag{DC-u}\label[ineq]{eq:basic_drift:drift_bounded_from_below}
            \Big(\randomProcess[\myIndex] - \E{\randomProcess[\myIndex + 1]}[\filtration[\myIndex]][\big]\Big) \cdot \indicator{\myIndex < \stoppingTimeOther} \geq \drift_\myIndex \cdot \indicator{\myIndex < \stoppingTimeOther} .
        \end{equation}
        Then
        \begin{equation}
            \label[ineq]{eq:basic_drift:stopping_time_bounded_from_above}
            \E{\sum\nolimits_{\myIndex \in [\stoppingTimeOther]} \drift_{\myIndex - 1}}[\filtration[0]] \leq \randomProcess[0] - \E{\randomProcess[\stoppingTimeOther]}[\filtration[0]][\big] .
        \end{equation}

        \item \label[case]{label:basic_drift:lower_bound}
        Assume that $\E{\stoppingTimeOther} < \infty$, and that there is a $c \in \R$ such that for all $\myIndex \in \N$, it holds that $\E{|\randomProcess[\myIndex] - \randomProcess[\myIndex + 1] - \drift_\myIndex|}[\filtration[\myIndex]] \cdot \indicator{\myIndex < \stoppingTimeOther} \leq c$, and that for all $\myIndex \in \N$, it holds that
        \begin{equation}
            \tag{DC-l}\label[ineq]{eq:basic_drift:drift_bounded_from_above}
            \Big(\randomProcess[\myIndex] - \E{\randomProcess[\myIndex + 1]}[\filtration[\myIndex]][\big]\Big) \cdot \indicator{\myIndex < \stoppingTimeOther} \leq \drift_\myIndex \cdot \indicator{\myIndex < \stoppingTimeOther} .
        \end{equation}
        Then
        \[
        \E{\sum\nolimits_{\myIndex \in [\stoppingTimeOther]} \drift_{\myIndex - 1}}[\filtration[0]] \geq \randomProcess[0] - \E{\randomProcess[\stoppingTimeOther]}[\filtration[0]][\big] .\qedhere
        \]
    \end{enumerate}
\end{theorem}

\begin{proof}
    Let \sequence{\randomProcessThird}[\myIndex][\N] be such that for all $\myIndex \in \N$, it holds that $\randomProcessThird[\myIndex] = \randomProcess[\myIndex] + \sum\nolimits_{k \in [\myIndex]} \drift_{k - 1}$.
    We aim to apply \Cref{thm:optional_stopping} to the stopped process \sequence{\randomProcessOther}[\myIndex][\N] of~\randomProcessThird, and to~\stoppingTimeOther and~\filtration.
    To this end, for all $\myIndex \in \N$, let
    \[
    \randomProcessOther[\myIndex] = \randomProcessThird[\myIndex] \cdot \indicator{\myIndex \leq \stoppingTimeOther} + \randomProcessThird[\stoppingTimeOther] \cdot \indicator{\myIndex > \stoppingTimeOther} .
    \]
    Note that, all by assumption,~\stoppingTimeOther is a stopping time with respect to~\filtration, and~\randomProcessOther is integrable and adapted to~\filtration, due to $(\randomProcess[\myIndex] \cdot \indicator{\myIndex \leq S})_{\myIndex \in \N}$ and $(\drift_\myIndex \cdot \indicator{\myIndex < S})_{\myIndex \in \N}$ having these properties.

    We are left to show that~\randomProcessOther is a super- or a submartingale with additional properties.
    To this end, we first consider the expected change of~\randomProcessOther.
    Let $\myIndex \in \N$.
    Noting that~\stoppingTimeOther is integer, observe that
    \begin{align*}
        &\randomProcessOther[\myIndex] - \E{\randomProcessOther[\myIndex + 1]}[\filtration[\myIndex]]\\
        &\quad= \randomProcessThird[\myIndex] \cdot \overbrace{\indicator{\myIndex \leq \stoppingTimeOther}}^{\mathclap{= \indicator{\myIndex = \stoppingTimeOther} + \indicator{\myIndex < \stoppingTimeOther}}}
            + \randomProcessThird[\stoppingTimeOther] \cdot \indicator{\myIndex > \stoppingTimeOther}\\
            &\qquad- \E{\randomProcessThird[\myIndex + 1] \cdot \overbrace{\indicator{\myIndex + 1 \leq \stoppingTimeOther}}^{\mathclap{= \indicator{\myIndex < \stoppingTimeOther}}}
            + \randomProcessThird[\stoppingTimeOther] \cdot \indicator{\myIndex + 1 > \stoppingTimeOther}}[\filtration[\myIndex]][\big]\\
        &\quad= \big(\randomProcessThird[\myIndex] - \E{\randomProcessThird[\myIndex + 1]}[\filtration[\myIndex]]\big) \cdot \indicator{\myIndex < \stoppingTimeOther}
            + \randomProcessThird[\myIndex] \cdot \indicator{\myIndex = \stoppingTimeOther}\\
            &\qquad+ \randomProcessThird[\stoppingTimeOther] \cdot \E{\overbrace{\indicator{\myIndex > \stoppingTimeOther} - \indicator{\myIndex + 1 > \stoppingTimeOther}}^{= -\indicator{\myIndex = \stoppingTimeOther}}}[\filtration[\myIndex]][\big]\\
        &\quad= \big(\randomProcessThird[\myIndex] - \E{\randomProcessThird[\myIndex + 1]}[\filtration[\myIndex]]\big) \cdot \indicator{\myIndex < \stoppingTimeOther}\\
        &\quad= \big(\randomProcess[\myIndex] - \E{\randomProcess[\myIndex + 1]}[\filtration[\myIndex]] - \drift_\myIndex\big) \cdot \indicator{\myIndex < \stoppingTimeOther} \numberthis\label{eq:basic_drift:potential_difference} .
    \end{align*}
    Whether~\randomProcessOther is a super- or a submartingale is determined by the sign of \cref{eq:basic_drift:potential_difference}, noting that $\indicator{\myIndex < \stoppingTimeOther} \geq 0$.
    To this end, we consider the two cases of \Cref{thm:basic_drift} separately.

    \textbf{\Cref{label:basic_drift:upper_bound}.}
    By \cref{eq:basic_drift:drift_bounded_from_below}, it follows that \cref{eq:basic_drift:potential_difference} is greater or equal to zero.
    Thus,~\randomProcessOther is a supermartingale.
    Further, due to the assumptions of \Cref{thm:basic_drift} \cref{label:basic_drift:upper_bound}, \randomProcessOther is non-negative.
    Applying \Cref{thm:optional_stopping} \cref{label:optional_stopping:supermartingale}, by the linearity of expectation and noting that $S \geq 0$, we get
    \begin{align}
        \label{eq:basic_drift:result}
        \E{\randomProcess[\stoppingTimeOther]}[\filtration[0]] + \E{\sum\nolimits_{\myIndex \in [S]} \drift_{\myIndex - 1}}[\filtration[0]] &= \E{\randomProcessOther[\stoppingTimeOther]}[\filtration[0]]\\
        \notag
        &\leq \randomProcessOther[0] = \randomProcess[0] .
    \end{align}
    Subtracting by $\E{\randomProcess[\stoppingTimeOther]}[\filtration[0]]$ proves this case.

    \textbf{\Cref{label:basic_drift:lower_bound}.}
    Similar to the previous case, by \cref{eq:basic_drift:drift_bounded_from_above}, it follows that \cref{eq:basic_drift:potential_difference} is less or equal to zero.
    Thus,~\randomProcessOther is a submartingale.
    By the assumptions of \Cref{thm:basic_drift} \cref{label:basic_drift:lower_bound}, it follows that~\randomProcessOther has uniformly bounded expected differences.
    Recalling that, by assumption, $\E{\stoppingTimeOther} < \infty$ in this case and by applying \Cref{thm:optional_stopping} \cref{label:optional_stopping:submartingale}, the same result as in \cref{eq:basic_drift:result} follows but with the inequality sign flipped, concluding this case and thus the proof.
\end{proof}

An application of \Cref{thm:basic_drift} is the generalization of Wald's equation below.
In contrast to more common statements of Wald's equation, this version neither assumes that the individual summands follow the same law nor that they are independent.
In the context of run time analysis, it is useful when aiming to split the run time into phases.
Instead of having to calculate the entire expected run time all at once, the theorem states that it is sufficient to calculate the expected run time of each phase (and \emph{then} add them).
The latter is typically far easier, especially since the theorem allows to choose a filtration.
This result was recently applied in the analysis of infection processes, where it was used in order to translate statements about discrete time to continuous time~\cite{FriedrichGKKP22Epidemics}.

\begin{theorem}[Generalized Wald's equation]
    \label{thm:walds_equation}
    Let~\stoppingTimeOther be an integrable stopping time with respect to a filtration \sequence{\filtration}[\myIndex][\N].
    Further, let \sequence{\randomProcess}[\myIndex][\N_{> 0}] be a random process over~$\RExtended_{\geq 0}$ such that $\sum_{\myIndex \in [\stoppingTimeOther]} \randomProcess[\myIndex]$ is integrable and such that there exists a $c \in \R$ such that for all $\myIndex \in \N$, it holds that $\E{\big|\randomProcess[\myIndex + 1] - \E{\randomProcess[\myIndex + 1]}[\filtration[\myIndex]]\big|}[\filtration[\myIndex]] \cdot \indicator{\myIndex < \stoppingTimeOther} \leq c$.
    Then $\E{\sum_{\myIndex \in [\stoppingTimeOther]} \randomProcess[\myIndex]}[\filtration[0]] = \E{\sum_{\myIndex \in [\stoppingTimeOther]} \E{\randomProcess[\myIndex]}[\filtration[\myIndex - 1]]}[\filtration[0]]$.
\end{theorem}

Checking that $\sum_{\myIndex \in [\stoppingTimeOther]} \randomProcess[\myIndex]$ is integrable can prove challenging in certain cases.
However, if this sum can be expressed as a single stopping time, \Cref{thm:basic_drift} \cref{label:basic_drift:upper_bound} might be used to show that the sum is integrable.

\begin{proof}[Proof of \Cref{thm:walds_equation}]
    Let \sequence{\transformedProcess}[\myIndex][\N] be such that for all $\myIndex \in \N$, it holds that $\transformedProcess[\myIndex] = \E{\sum_{j = \myIndex + 1}^{\stoppingTimeOther} \randomProcess[j]}[\filtration[\myIndex]]$.
    Further, let \sequence{\drift}[\myIndex][\N] be such that for all $\myIndex \in \N$, it holds that $\drift_\myIndex = \E{\randomProcess[\myIndex + 1]}[\filtration[\myIndex]]$.
    Note that $(\transformedProcess[\myIndex] \cdot \indicator{\myIndex \leq \stoppingTimeOther})_{\myIndex \in \N}$ and $(\drift_\myIndex \cdot \indicator{\myIndex < \stoppingTimeOther})_{\myIndex \in \N}$ are adapted to~\filtration due to the conditional expectation and that they are integrable since $\sum_{\myIndex \in [\stoppingTimeOther]} \randomProcess[\myIndex]$ is integrable and~\randomProcess is non-negative.

    We aim to apply both cases of \Cref{thm:basic_drift} to~\transformedProcess and~\drift with~\stoppingTimeOther and~\filtration.
    To this end, we note that for all $\myIndex \in \N$, by the tower rule and the linearity of the conditional expectation, we have
    \begin{align*}
        &\big(\transformedProcess[\myIndex] - \E{\transformedProcess[\myIndex + 1]}[\filtration[\myIndex]]\big) \cdot \indicator{\myIndex < \stoppingTimeOther}\\
       	&\quad= \bigg(\E{\sum\nolimits_{j = \myIndex + 1}^{\stoppingTimeOther} \randomProcess[j]}[\filtration[\myIndex]]\\
           &\qquad\quad- \E{\E{\sum\nolimits_{j = \myIndex + 2}^{\stoppingTimeOther} \randomProcess[j]}[\filtration[\myIndex + 1]]}[\filtration[\myIndex]][\bigg]\bigg) \cdot \indicator{\myIndex < \stoppingTimeOther}\\
        &\quad= \bigg(\E{\sum\nolimits_{j = \myIndex + 1}^{\stoppingTimeOther} \randomProcess[j]}[\filtration[\myIndex]]\\
            &\qquad\quad- \E{\sum\nolimits_{j = \myIndex + 2}^{\stoppingTimeOther} \randomProcess[j]}[\filtration[\myIndex]]\bigg) \cdot \indicator{\myIndex < \stoppingTimeOther}\\
        &\quad= \E{\randomProcess[\myIndex + 1]}[\filtration[\myIndex]] \cdot \indicator{\myIndex < \stoppingTimeOther} = \drift_\myIndex \cdot \indicator{\myIndex < \stoppingTimeOther}\ ,
    \end{align*}
    which shows that \cref{eq:basic_drift:drift_bounded_from_above,eq:basic_drift:drift_bounded_from_below} are satisfied.
    We now apply both cases of \Cref{thm:basic_drift} separately in order to show both directions of the equality of \Cref{thm:walds_equation}.

    \textbf{Case 1:}
    Note that $(\transformedProcess[\myIndex] \cdot \indicator{\myIndex \leq \stoppingTimeOther})_{\myIndex \in \N}$ is non-negative because~\randomProcess is.
    By \Cref{thm:basic_drift} \cref{label:basic_drift:upper_bound},
    \begin{align*}
        &\E{\sum\nolimits_{\myIndex \in [\stoppingTimeOther]} \E{\randomProcess[\myIndex]}[\filtration[\myIndex - 1]]}[\filtration[0]]
        = \E{\sum\nolimits_{\myIndex \in [\stoppingTimeOther]} \drift_{\myIndex - 1}}[\filtration[0]]\\
        &\quad\leq \transformedProcess[0] - \E{\transformedProcess[\stoppingTimeOther]}[\filtration[0]][\big]
        = \E{\sum\nolimits_{\myIndex \in [\stoppingTimeOther]} \randomProcess[\myIndex]}[\filtration[0]] .
    \end{align*}

    \textbf{Case 2:}
    By assumption, \stoppingTimeOther is integrable, and there is a $c \in \R$ such that for all $\myIndex \in \N$, it holds that
    \begin{align*}
        &\E{|\transformedProcess[\myIndex] - \transformedProcess[\myIndex + 1] - \drift_\myIndex|}[\filtration[\myIndex]] \cdot \indicator{\myIndex < \stoppingTimeOther}\\
        &\quad= \E{\left|\sum\nolimits_{j = \myIndex + 1}^{\stoppingTimeOther} \randomProcess[j] - \sum\nolimits_{j = \myIndex + 2}^{\stoppingTimeOther} \randomProcess[j] - \E{\randomProcess[\myIndex + 1]}[\filtration[\myIndex]]\right|}[\filtration[\myIndex]]\\
            &\qquad\cdot \indicator{\myIndex < \stoppingTimeOther}\\
        &\quad= \E{\big|\randomProcess[\myIndex + 1] - \E{\randomProcess[\myIndex + 1]}[\filtration[\myIndex]]\big|}[\filtration[\myIndex]] \cdot \indicator{\myIndex < \stoppingTimeOther}
        \leq c .
    \end{align*}

    By \Cref{thm:basic_drift} \cref{label:basic_drift:lower_bound} and the same calculations as at the end of the previous case but with an inverted inequality sign, we get a matching lower bound and conclude the proof.
\end{proof}


\section{Theoretical Results for \rls}
\label{sec:eas_on_plateaus}
We analyze the expected run time of \rls on \plateau (\Cref{thm:eas_on_plateaus}) and on \asymmetricPlateau (\Cref{thm:rls_run_time_on_asymmetric_plateau}).
We exploit the similarities between both functions, deriving the result for \asymmetricPlateau based on the result for \plateau.
In fact, we prove a very general statement (\Cref{thm:restart_argument}), which shows how to derive the expected run time for a great range of algorithms optimizing \asymmetricPlateau, given their expected run time on \plateau and some additional, related expected run times.
This result is a special case of an even more general theorem (\Cref{thm:general_restart_argument}), which decomposes an expected stopping time into smaller intervals, which are easier to analyze.
Thus, overall, deriving a good bound on the expected run time on \plateau is essential.
Last, we show how our result on \asymmetricPlateau implies run time bounds for the classical \onemax when transformed by neutrality of the W-model (\Cref{cor:rls_w-model}).

When \rls optimizes \plateau, before finding an optimal solution for the first time, it performs a random walk, as it always replaces the solution from the previous iteration with the new solution.
This random walk is biased toward solutions with as many~$0$s as~$1$s (the \emph{center}), since, in order to create a new solution, \rls flips one bit uniformly at random from the solution from the previous iteration and any imbalance in the number of~$0$s and~$1$s favors flipping those bits that occur in a higher quantity.
Thus, the number of possible ways of an individual to get to the center compared to those toward an individual with all~$1$s or all~$0$s grows exponentially in the distance to the center.
Since the optima of \plateau are off-center, this results in the expected run time growing exponentially with respect to a basis that we call $\basisRLS \in \R_{> 1}$, which is \emph{not} necessarily bounded away from~$1$ by a constant.

Let $n \in \N_{> 0}$ be even, and let $\plateauDistance \in [n/2]$.
The expected run time of \rls on \plateau[\plateauDistance] majorly depends on
\begin{align}
    \label{eq:basisRLS}
    \basisRLS &= \frac{3\plateauDistance\big(n + 2(\plateauDistance - 1)\big)}{3\plateauDistance\big(n - 2(\plateauDistance - 1)\big) - 2n}
    =1 + \frac{2\big(n + 6\plateauDistance(\plateauDistance - 1)\big)}{3\plateauDistance\big(n - 2(\plateauDistance - 1)\big) - 2n}\ ,
\end{align}
which is greater than~$1$.
We discuss this term and its impact on the expected run time of \rls after stating our main results.

\begin{theorem}
    \label{thm:eas_on_plateaus}
    Let $n \in \N_{> 0}$ be even, $\plateauDistance \in [n / 2]$, let~\basisRLS be as in \cref{eq:basisRLS}, and let $h\colon \{0, 1\}^n \to \R_{\geq 0}$ with
    \begin{align*}
        \individual \mapsto
        \begin{cases}
            \basisRLS^\plateauDistance - \basisRLS^{\max\{|\individual|_0, |\individual|_1\} - n/2} & : \max\{|\individual|_0, |\individual|_1\} < \frac{n}{2} + \plateauDistance,\\
            0                                                                                       & : \textrm{otherwise.}
        \end{cases}
    \end{align*}
    Let~\individual[0] denote the (possibly random) initial individual of \rls.
    The expected run time of \rls on \plateau[\plateauDistance], conditional on~\individual[0], is at most $3\plateauDistance \cdot h(\individual[0]) / (\basisRLS - 1)$.
\end{theorem}

\begin{theorem}
    \label{thm:rls_run_time_on_asymmetric_plateau}
    Let $n \in \N_{> 0}$ be even, $\plateauDistance \in [n / 2]$, and let~\basisRLS be as in \cref{eq:basisRLS}.
    For sufficiently large~$n$, the expected run time of \rls on \asymmetricPlateau[\plateauDistance] with uniform initialization is at most $6\plateauDistance \cdot (\basisRLS^\plateauDistance - 1) / (\basisRLS - 1) + n\big(1 + \ln(\plateauDistance)\big)/2$.
\end{theorem}

\Cref{thm:eas_on_plateaus} shows that there is a \emph{drastic} change in the expected run time of \rls on \plateau with respect to~\plateauDistance.
In order to see this, recall \cref{eq:basisRLS} and note that $h(\individual[0]) \leq \lambda^\plateauDistance$.
If $\plateauDistance = \bigTheta{1}$, then $\lambda = 1 + \bigTheta{1}$, and the expected run time is constant.
If $\plateauDistance = \smallOmega{1} \cap \bigO{\sqrt{n}}$, then $\lambda = 1 + \bigO{1/\plateauDistance}$, and the expected run time is \bigO{\plateauDistance^2}, as $\lambda^\plateauDistance \leq \eulerE^{\bigO{1}}$.
Especially, if $\plateauDistance = \bigO{\sqrt{n}}$, then the expected run time is at most linear in~$n$.
If $\plateauDistance = \smallOmega{\sqrt{n}} \cap \smallO{n}$, then $\lambda = 1 + \bigO{\plateauDistance/n}$, and the expected run time is $n \eulerE^{\bigO{\plateauDistance^2/n}}$, which is superlinear in~$n$ and  even superpolynomial in~$n$ for $r = \smallOmega{\sqrt{n \log n}}[\big]$.
Last, if there is a $c = \bigTheta{1}$ with $c < 1/2$ and an $s \in \N$ with $s \leq cn$ such that $r = n/2 - s$, then $\lambda = 1 + \bigO{n/(s + 1)}[\big]$, and the expected run time is $(s + 1) \eulerE^{\bigO{n^2/(s + 1)}}$, which is at least exponential in~$n$.

This drastic increase in the expected run time carries over to \Cref{thm:rls_run_time_on_asymmetric_plateau}, whose bound we compare to an empirically determined one, as depicted in \Cref{fig:run_time_varying_r} for $n = 10\,000$.
Due to the steep growth of the run time in~\plateauDistance, we only get empirical results up to $\plateauDistance = 170 = 1.7 \sqrt{n}$.
Nonetheless, the empirical expected run time suggests an exponential growth in~\plateauDistance.
However, the growth does not seem to be as drastic as the theoretical upper bound of \Cref{thm:rls_run_time_on_asymmetric_plateau} suggests, which is larger than the empirical average by more than a constant factor.

\begin{figure}
    \centering
    \includegraphics[width = \columnwidth]{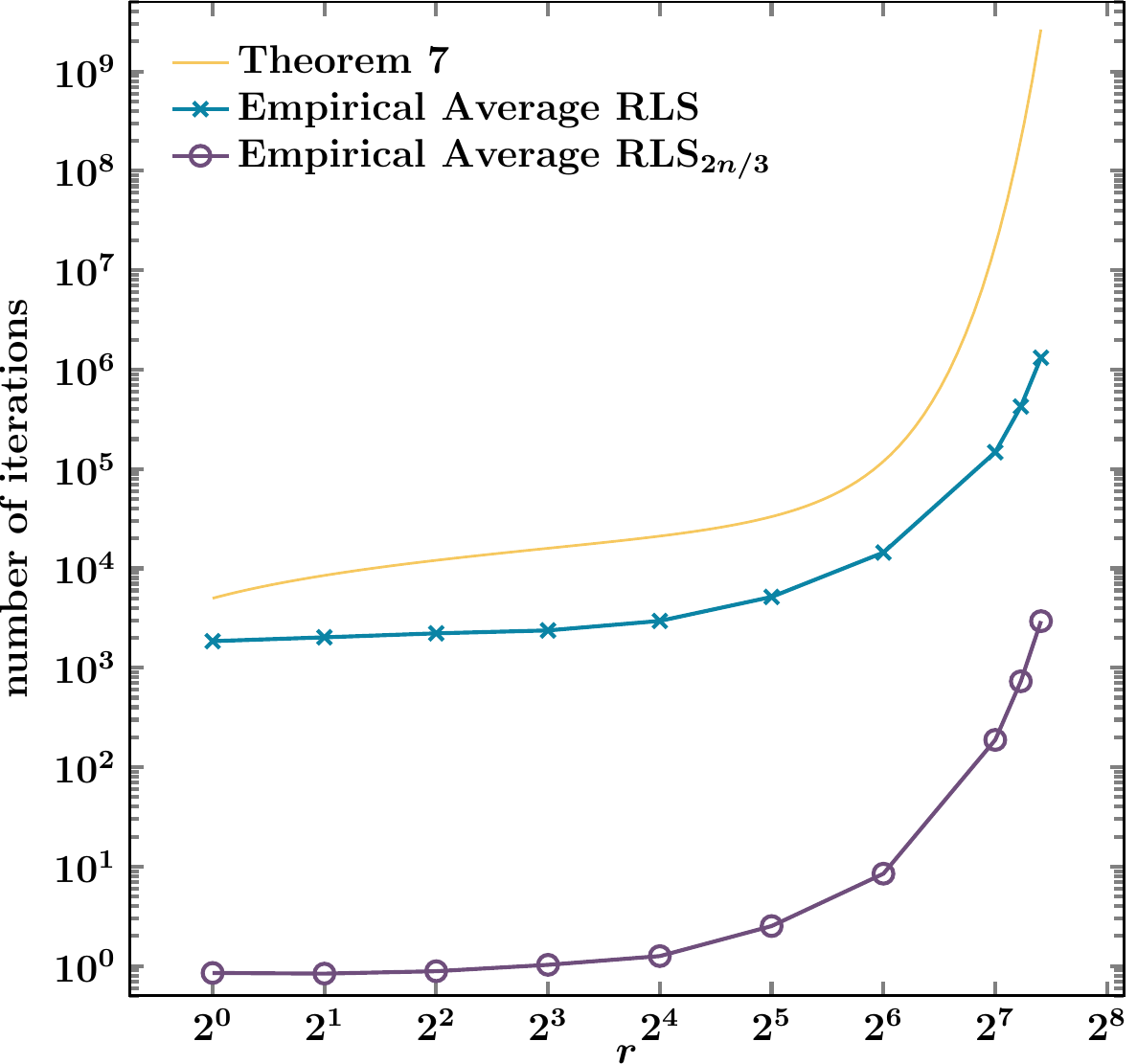}
    \caption{\label{fig:run_time_varying_r}
        The theoretical (top, orange curve) and empirical expected run time of \rls (center, blue curve with~$\times$) and \rls[2n/3] (bottom, purple curve with~$\circ$) on \asymmetricPlateau[\plateauDistance] for $n = 10\,000$ and $r \in \{2^i \mid i \in \{0\} \cup [7]\} \cup \{150, 170\}$.
        The theoretical run time is the upper bound from \Cref{thm:rls_run_time_on_asymmetric_plateau}.
        For the empirical run times, for each value of~$r$ stated above, the average of~$1000$ independent runs of~\rls and of~\rls[2n/3] with uniform initialization is depicted.
        For more information regarding \rls[2n/3], please see \Cref{sec:empirical_run_time_varying_n_and_k}; the value $2n/3$ is near-optimal, as shown in \Cref{fig:empirical_run_time_varying_n_and_k}, resulting in a drastically improved run time.
    }
\end{figure}

\subsection{Run Time Results on \texorpdfstring{\plateau}{HasMajority}}
\label{sec:plateau_run_times}

In our proof of \Cref{thm:eas_on_plateaus}, we aim to apply the additive-drift theorem (\Cref{thm:additive_drift}).
To this end, letting~\individual denote the currently best individual of \rls in each iteration, we consider the distance $n/2 + \plateauDistance - \max\{|\individual|_0, |\individual|_1\}$ of the majority of bits of \individual to the optimum value $n/2 + \plateauDistance$ of \plateau[\plateauDistance].
Since it becomes less likely to decrease this distance the smaller it is, we choose a potential function that scales the space exponentially, counteracting this decline in probability.

\begin{proof}[Proof of \Cref{thm:eas_on_plateaus}]
    For all $\timePoint \in \N$, let $\randomProcess[\timePoint] = \max\{|\individual[\timePoint]|_0, |\individual[\timePoint]|_1\}$, and let $(\filtration[\timePoint])_{\timePoint \in \N}$ denote the natural filtration of~\randomProcess.
    In addition, let $\potentialFunction\colon [n/2 .. n] \to \R_{\geq 0}$ with
    \[
        m \mapsto
        \begin{cases}
            \basisRLS^\plateauDistance - \basisRLS^{m - n/2} & : m < n/2 + \plateauDistance\,,\\
            0                                                & : \textrm{otherwise.}
        \end{cases}
    \]
    Last, let $\stoppingTime = \inf\{\timePoint \in \N \mid \potentialFunction(\randomProcess[\timePoint]) = 0\}$, and note that~\stoppingTime is a stopping time with respect to~\filtration and that it denotes the first point in time such that \rls found an optimum of \plateau[\plateauDistance], as $\randomProcess[\stoppingTime] \geq n/2 + \plateauDistance$.
    We show that~$\potentialFunction(\randomProcess)$ satisfies \cref{eq:additive-drift:drift_condition} with~$\filtration$ and~\stoppingTime.
    \Cref{thm:additive_drift} then yields the desired bound.

    Let $\timePoint \in \N$, and assume that $\timePoint < \stoppingTime$, as \cref{eq:additive-drift:drift_condition} is trivially satisfied otherwise.
    Note that~\randomProcess[\timePoint] increases by~$1$ if a bit that is not the strict majority in~\individual[\timePoint] is chosen to be flipped, which happens with probability $(n - \randomProcess[\timePoint])/n$ if $\randomProcess[\timePoint] > n/2$, and with probability~$1$ if $\randomProcess[\timePoint] = n/2$.
    With the remaining probability, \randomProcess[\timePoint] decreases by~$1$.

    For $\randomProcess[\timePoint] = n/2$, we get
    \begin{align}
        \label[ineq]{eq:eas_on_plateaus:rls:case1:drift_bound}
        \potentialFunction(\randomProcess[\timePoint]) - \E{\potentialFunction(\randomProcess[\timePoint + 1])}[\filtration[\timePoint]]
            &= \basisRLS^\plateauDistance - 1 - \left(\basisRLS^\plateauDistance - \basisRLS\right)
            = \basisRLS - 1  .
    \end{align}
    For the remaining case $\randomProcess[\timePoint] \in [n/2 + 1 .. n/2 + \plateauDistance - 1]$, we get
    \begin{align}
        \notag
        &\potentialFunction(\randomProcess[\timePoint]) - \E{\potentialFunction(\randomProcess[\timePoint + 1])}[\filtration[\timePoint]]\\
        \notag
        &\quad= \basisRLS^\plateauDistance - \basisRLS^{\randomProcess[\timePoint] - n/2} - \left(\basisRLS^\plateauDistance - \basisRLS^{\randomProcess[\timePoint] - 1 - n/2}\right) \frac{\randomProcess[\timePoint]}{n}\\
            \notag
            &\qquad- \left(\basisRLS^\plateauDistance - \basisRLS^{\randomProcess[\timePoint] + 1 - n/2}\right) \frac{n - \randomProcess[\timePoint]}{n}\\
        \label{eq:eas_on_plateaus:rls:case2:drift_estimate}
        &\quad= \basisRLS^{\randomProcess[\timePoint] - n/2} \Big(\underbrace{\frac{1}{\basisRLS} \cdot \frac{\randomProcess[\timePoint]}{n} + \basisRLS \cdot \frac{n - \randomProcess[\timePoint]}{n}}_{\eqdef A} - 1\Big) .
    \end{align}
    Let $b = 2(\plateauDistance - 1)$, $\basisRLSNumerator = 3\plateauDistance(n - b)(n + b)/\big(3\plateauDistance(n - b) - 2n\big)$, and note that $\basisRLS = \basisRLSNumerator/(n - b)$.
    Since~\basisRLS is greater than~$1$, it follows that~$A$ is decreasing in~$\randomProcess[\timePoint]$ and thus minimal for $\randomProcess[\timePoint] = n/2 + r - 1 = (n + b)/2$.
    Substituting $\basisRLS = \basisRLSNumerator/(n - b)$, this yields
    \begin{align}
        \label[ineq]{eq:eas_on_plateaus:rls:case2:drift_constant_estimate}
        A &\geq \frac{n - b}{\basisRLSNumerator} \cdot \frac{n + b}{2n} + \frac{\basisRLSNumerator}{n - b} \cdot \frac{n - b}{2n}\\
        \notag
        &= \frac{(n - b)(n + b) + \basisRLSNumerator^2}{2 \basisRLSNumerator n}
        = \underbrace{\frac{n^2 - b^2}{2 \basisRLSNumerator n}}_{\eqdef B} + \frac{\basisRLSNumerator}{2n} .
    \end{align}
    Noting that $\basisRLSNumerator = 3\plateauDistance(n^2 - b^2)/\big(3\plateauDistance(n - b) - 2n\big)$, we get
    \begin{align*}
        B &= \frac{3\plateauDistance(n - b) - 2n}{2n \cdot 3\plateauDistance}
        = \frac{n - b}{2n} - \frac{1}{3\plateauDistance} .
    \end{align*}
    Substituting this back into \cref{eq:eas_on_plateaus:rls:case2:drift_constant_estimate} yields
    \begin{align}
        \label[ineq]{eq:eas_on_plateaus:rls:case2:drift_constant_estimate_2}
        A \geq \underbrace{\frac{\basisRLSNumerator + n - b}{2n}}_{\eqdef C} - \frac{1}{3\plateauDistance} .
    \end{align}
    Again, by substituting $\basisRLSNumerator = 3\plateauDistance(n - b)(n + b)/(3\plateauDistance(n - b) - 2n)$, we get
    \begin{align*}
        C &= \frac{1}{2n}\left(\frac{3\plateauDistance(n - b)(n + b)}{3\plateauDistance(n - b) - 2n} + \frac{(n - b) \big(3\plateauDistance(n - b) - 2n\big)}{3\plateauDistance(n - b) - 2n}\right)\\
        &= \frac{n - b}{2n} \cdot \frac{3\plateauDistance(n + b + n - b) - 2n}{3\plateauDistance(n - b) - 2n}
        = \frac{3\plateauDistance(n - b) - (n - b)}{3\plateauDistance(n - b) - 2n}\\
        &= 1 + \frac{n + b}{3\plateauDistance(n - b) - 2n}
        = 1 + \frac{1}{3\plateauDistance} \cdot \frac{3\plateauDistance(n + b)}{3\plateauDistance(n - b) - 2n}\\
        &= 1 + \frac{\basisRLSNumerator}{n - b} \cdot \frac{1}{3\plateauDistance} .
    \end{align*}
    Substituting this back into \cref{eq:eas_on_plateaus:rls:case2:drift_constant_estimate_2} and recalling that $\basisRLS = \basisRLSNumerator/(n - b)$ yields $A \geq 1 + (\basisRLS - 1)/(3\plateauDistance)$.
    Substituting this bound into \cref{eq:eas_on_plateaus:rls:case2:drift_estimate} ultimately yields
    \begin{align}
        \label[ineq]{eq:eas_on_plateaus:rls:case2:drift_bound}
        \potentialFunction(\randomProcess[\timePoint]) - \E{\potentialFunction(\randomProcess[\timePoint + 1])}[\filtration[\timePoint]]
            \geq \basisRLS^{\randomProcess[\timePoint] - n/2} \cdot (\basisRLS - 1) \cdot \frac{1}{3\plateauDistance} .
    \end{align}

    Applying \Cref{thm:additive_drift} with $\drift = (\basisRLS - 1)/(3\plateauDistance)$ being the minimum of \cref{eq:eas_on_plateaus:rls:case1:drift_bound,eq:eas_on_plateaus:rls:case2:drift_bound} and noting that $\E{g(\randomProcess[\stoppingTime])}[\filtration[0]] = 0$ concludes the proof.
\end{proof}

\subsection{Run Time Results on \texorpdfstring{\asymmetricPlateau}{Majority}}
\label{sec:run_time_on_asymmetric_plateau}

\begin{figure}
    \centering
    \includegraphics[width = \columnwidth]{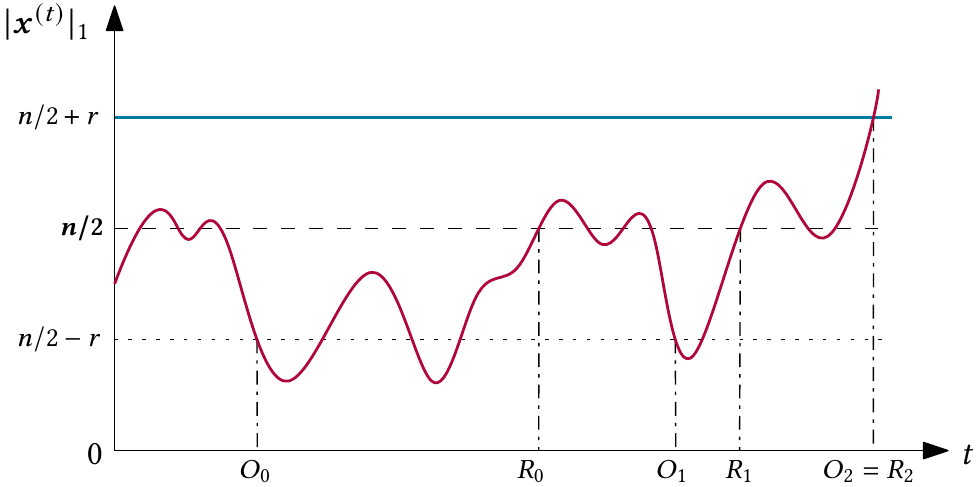}
    \caption{
        \label{fig:restart_argument}
        A visualization of the definitions of the random variables $(\returnTimePointsEnd[i])_{i \in \N}$ and $(\returnTimePointsStart[i])_{i \in \N}$ as defined in \cref{eq:definition_of_the_return_times}, needed for \Cref{thm:restart_argument}.
        The red wavy line depicts an example of how the number of~$1$s in the best solution of an EA (given by the sequence $(\individual[\timePoint])_{\timePoint \in \N}$) changes over time (\timePoint).
        For the sake of readability, we depict this line as continuous although it is discrete and can have jumps.
        The stopping times $(\returnTimePointsEnd[i])_{i \in \N}$ and $(\returnTimePointsStart[i])_{i \in \N}$ always alternate, by definition.
    }
\end{figure}

We show how to derive a run time bound of an EA~\myAlgorithm for \asymmetricPlateau based on a run time bound for \plateau (\Cref{thm:restart_argument}).
The argument is based on the observation that whenever~\myAlgorithm finds an optimum of \plateau, this solution can also be an optimum of \asymmetricPlateau.
If this is \emph{not} the case, that is, if the current solution has a majority of~$0$s, we wait until~\myAlgorithm gets back to a solution of at least $n/2$ $1$s.
Then, we wait again until it finds an optimum of \plateau, repeating this argument until~\myAlgorithm finds an optimum of \asymmetricPlateau.

More formally, we define the following two interleaving sequences of stopping times, which we also visualize in \Cref{fig:restart_argument}:
\begin{align*}
    \numberthis\label{eq:definition_of_the_return_times}
    \returnTimePointsEnd[0] &= \inf\{\timePoint \in \N \mid \plateau[\plateauDistance](\individual[\timePoint]) = 1\}\ ,\\
    \intertext{ as well as, for all $i \in \N$,}
    \returnTimePointsStart[i] &= \inf\{\timePoint \in \N \mid \timePoint \geq \returnTimePointsEnd[i] \land |\individual[\timePoint]|_1 \geq n/2\}\ \textrm{, and}\\
    \returnTimePointsEnd[i + 1] &= \inf\{\timePoint \in \N \mid \timePoint > \returnTimePointsStart[i] \land \plateau[\plateauDistance](\individual[\timePoint]) = 1\} .
\end{align*}
The sequence~\returnTimePointsEnd denotes the points in time (interleaved by~\returnTimePointsStart) where~\myAlgorithm can have found an optimum of \asymmetricPlateau, as it found an optimum of \plateau.
In the context of proving a run time bound for \asymmetricPlateau, \returnTimePointsEnd denotes the points in time from where we wait until~\myAlgorithm returns to an individual with at least as many~$1$s as~$0$s in order to repeat our argument.
The sequence~\returnTimePointsStart denotes those points in time where we repeat our argument.
\Cref{thm:walds_equation} formalizes this argument and yields a bound on the overall expected run time on \asymmetricPlateau.

Since the idea of this argument holds for any interleaved assortment of stopping times, we first prove a more general theorem (\Cref{thm:general_restart_argument}) before we state the result specific to our setting (\Cref{thm:restart_argument}).
This theorem generalizes the idea of the stopping times from \cref{eq:definition_of_the_return_times}, but it does not specify their interpretation.

\begin{theorem}[Decomposition argument]
    \label{thm:general_restart_argument}
    Let \sequence{\filtration}[i][\N] be a filtration, let~\stoppingTimeOther be a stopping time with respect to~\filtration, and let $(\returnTimePointsEnd[i])_{i \in \N}$ and $(\returnTimePointsStart[i])_{i \in \N}$ be stopping times with respect to~\filtration such that for all $i \in \N$, it holds that $\returnTimePointsEnd[i] \leq \returnTimePointsStart[i] < \returnTimePointsEnd[i + 1]$.
    Moreover, let $\stoppingTimeThird = \sup \{i \in \N \mid \returnTimePointsEnd[i] \leq \stoppingTimeOther\}$, and assume that there is a $\probabilityOfImmediateSuccess \in (0, 1]$ such that $\Pr{\returnTimePointsEnd[0] = \stoppingTimeOther} = \probabilityOfImmediateSuccess$.

    Assume that $\sum_{i \in [\stoppingTimeThird]} (\returnTimePointsEnd[i] - \returnTimePointsEnd[i - 1])$ and~\stoppingTimeThird are integrable as well as that there is a $c \in \R$ such that for all $i \in \N_{> 0}$, it holds that $\E{|\returnTimePointsEnd[i] - \E{\returnTimePointsEnd[i]}[\filtration[\returnTimePointsStart[i - 1]]]|}[\filtration[\returnTimePointsStart[i - 1]]] \leq c$ and $\E{|\returnTimePointsStart[i - 1] - \E{\returnTimePointsStart[i - 1]}[\filtration[\returnTimePointsEnd[i - 1]]]|}[\filtration[\returnTimePointsEnd[i - 1]]] \leq c$.

    Last, let~\eventForRestart denote the event that $U \geq 1$.
    Then
    \begin{align*}
        \E{\stoppingTimeOther} &= \E{\returnTimePointsEnd[0]} + (1 - \probabilityOfImmediateSuccess) \mathrm{E}\Big[\sum\nolimits_{i \in [\stoppingTimeThird]} \big(\E{\returnTimePointsEnd[i] - \returnTimePointsStart[i - 1]}[\filtration[\returnTimePointsStart[i - 1]]]\\
        &\hspace*{3.1 cm}+ \E{\returnTimePointsStart[i - 1] - \returnTimePointsEnd[i - 1]}[\filtration[\returnTimePointsEnd[i - 1]]]\big)\,\Big\vert\,\eventForRestart\Big]\,,
    \end{align*}
    where we interpret the second summand to be~$0$ if $1 - p_0 = 0$.
\end{theorem}

\begin{proof}
    We aim to express~\stoppingTimeOther as a sum of time intervals that reflect the repeats as outlined at the beginning of \Cref{sec:run_time_on_asymmetric_plateau}.
    To this end, for all $i \in \N_{> 0}$, let $\difference[i] = \returnTimePointsEnd[i] - \returnTimePointsEnd[i - 1]$.
    Note that this implies that $\stoppingTimeOther = \stoppingTime + \sum_{i \in [\stoppingTimeThird]} \difference[i]$, since $\stoppingTime = \returnTimePointsEnd[0]$, $\stoppingTimeOther = \returnTimePointsEnd[\stoppingTimeThird]$, and the sum is telescoping.

    By the linearity of expectation, it follows that $\E{\stoppingTimeOther} = \E{\stoppingTime} + \E{\sum_{i \in [\stoppingTimeThird]} \difference[i]}$.
    Note that $\E{\sum_{i \in [\stoppingTimeThird]} \difference[i]}[\overline{\eventForRestart}][\big] = 0$ as well as $\Pr{\eventForRestart} = 1 - \probabilityOfImmediateSuccess$.
    Thus, $\E{\sum_{i \in [\stoppingTimeThird]} \difference[i]} = (1 - \probabilityOfImmediateSuccess) \E{\sum_{i \in [\stoppingTimeThird]} \difference[i]}[\eventForRestart]$, which is~$0$ if $\Pr{\eventForRestart} = 0$.

    Let $i \in \N_{> 0}$, and note that, $\difference[i] = \returnTimePointsEnd[i] - \returnTimePointsEnd[i - 1] = (\returnTimePointsEnd[i] - \returnTimePointsStart[i - 1]) + (\returnTimePointsStart[i - 1] - \returnTimePointsEnd[i - 1])$.
    Thus, by the linearity of expectation,
    \begin{align*}
        \E{\sum\nolimits_{i \in [\stoppingTimeThird]} \difference[i]}[\eventForRestart]
        &= \E{\sum\nolimits_{i \in [\stoppingTimeThird]} (\returnTimePointsEnd[i] - \returnTimePointsStart[i - 1])}[\eventForRestart]\\
        &\quad+ \E{\sum\nolimits_{i \in [\stoppingTimeThird]} (\returnTimePointsStart[i - 1] - \returnTimePointsEnd[i - 1])}[\eventForRestart] .
    \end{align*}
    We aim to apply \Cref{thm:walds_equation} to both expected values on the right-hand side with~\stoppingTimeThird and a suitable filtration.
    To this end, let
    \[
        (\returnTimePoints{$i$})_{i \in \N_{> 0}} = (\returnTimePointsEnd[i + 1] - \returnTimePointsStart[i])_{i \in \N} \textrm{\,and\,}
        (\runTimePoints{$i$})_{i \in \N_{> 0}} = (\returnTimePointsStart[i] - \returnTimePointsEnd[i])_{i \in \N} ,
    \]
    and note that the indices of the new sequences start at~$1$ whereas they start at~$0$ for the old ones.
    Note that~\stoppingTimeThird is integrable by assumption.

    \textbf{Considering~\returnTimePoints{}:}
    Let $(\filtrationOther[i])_{i \in \N} = (\filtration[\inf\{\returnTimePointsStart[i], \stoppingTimeOther\}])_{i \in \N}$.
    Note that $(\returnTimePoints{$i$} \cdot \indicator{i \leq \stoppingTimeThird})_{i \in \N_{> 0}}$ is adapted to~\filtrationOther, as for all $i \in \N_{> 0}$, it holds that $\returnTimePointsStart[i - 1] \cdot \indicator{i \leq \stoppingTimeThird} < \returnTimePointsEnd[i] \cdot \indicator{i \leq \stoppingTimeThird} \leq \inf\{\returnTimePointsStart[i], \stoppingTimeOther\}$ and thus that $\returnTimePoints{$i$} \cdot \indicator{i \leq \stoppingTimeThird}$ is \filtrationOther[i]-measurable.

    Since $\sum_{i \in [\stoppingTimeThird]} (\returnTimePointsEnd[i] - \returnTimePointsEnd[i - 1]) = \sum_{i \in [\stoppingTimeThird]} \returnTimePoints{$i$} + \sum_{i \in [\stoppingTimeThird]} \runTimePoints{$i$}$ is integrable by assumption, and since~\returnTimePoints{} and~\runTimePoints{} are non-negative by the definition of~\returnTimePointsEnd and~\returnTimePointsStart, it follows that $\sum_{i \in [\stoppingTimeThird]} \returnTimePoints{$i$}$ is integrable.
    Further, since $(1)$ \returnTimePointsStart[i - 1] is \filtration[\returnTimePointsStart[i - 1]]-measurable, and $(2)$ by the assumption that there is a $c \in \R$ such that for all $i \in \N_{> 0}$, we have $\E{|\returnTimePointsEnd[i] - \E{\returnTimePointsEnd[i]}[\filtration[\returnTimePointsStart[i - 1]]]|}[\filtration[\returnTimePointsStart[i - 1]]] \leq c$,
    \begin{align*}
        &\E{|\returnTimePoints{$i$} - \E{\returnTimePoints{$i$}}[\filtration[\returnTimePointsStart[i - 1]]]|}[\filtration[\returnTimePointsStart[i - 1]]]\\
        &= \E{|\returnTimePointsEnd[i] - \returnTimePointsStart[i - 1] - \E{\returnTimePointsEnd[i] - \returnTimePointsStart[i - 1]}[\filtration[\returnTimePointsStart[i - 1]]]|}[\filtration[\returnTimePointsStart[i - 1]]]\\
        &\overset{(1)}{=} \E{|\returnTimePointsEnd[i] - \E{\returnTimePointsEnd[i]}[\filtration[\returnTimePointsStart[i - 1]]]|}[\filtration[\returnTimePointsStart[i - 1]]] \overset{(2)}{\leq} c .
    \end{align*}
    Since, for all $i \in \N_{> 0}$ such that $i \leq \stoppingTimeThird$, it holds that $\returnTimePointsStart[i - 1] < \stoppingTimeOther$, it follows for all $i \in \N_{> 0}$ that
    \begin{align*}
        &\E{|\returnTimePoints{$i$} - \E{\returnTimePoints{$i$}}[\filtrationOther[i - 1]]|}[\filtrationOther[i - 1]] \cdot \indicator{i - 1 < \stoppingTimeThird} \leq c .
    \end{align*}
    Applying \Cref{thm:walds_equation} and taking the expectations on both sides yields $\E{\sum\nolimits_{i \in [\stoppingTimeThird]} \returnTimePoints{$i$}}[\eventForRestart][\Big] = \E{\sum\nolimits_{i \in [\stoppingTimeThird]} \E{\returnTimePoints{$i$}}[\filtrationOther[i - 1]]}[\eventForRestart][\Big]$.

    \textbf{Considering~\runTimePoints{}:}
    This case is very similar to the previous one.
    Let $(\filtrationOther[i])_{i \in \N} = (\filtration[\inf\{\returnTimePointsEnd[i], \stoppingTimeOther\}])_{i \in \N}$.
    Note that $(\runTimePoints{$i$} \cdot \indicator{i \leq \stoppingTimeThird})_{i \in \N_{> 0}}$ is adapted to~\filtrationOther, as for all $i \in \N_{> 0}$, it holds that $\returnTimePointsEnd[i - 1] \cdot \indicator{i \leq \stoppingTimeThird} \leq \returnTimePointsStart[i - 1] \cdot \indicator{i \leq \stoppingTimeThird} < \inf\{\returnTimePointsEnd[i], \stoppingTimeOther\}$.
    Further, as in the previous case, since $\sum_{i \in [\stoppingTimeThird]} (\returnTimePointsEnd[i] - \returnTimePointsEnd[i - 1])$ is integrable by assumption, so is $\sum_{i \in [\stoppingTimeThird]} \runTimePoints{$i$}$.
    Last, by the assumption that there is a $c \in \R$ such that for all $i \in \N_{> 0}$, it holds that $\E{|\returnTimePointsStart[i - 1] - \E{\returnTimePointsStart[i - 1]}[\filtration[\returnTimePointsEnd[i - 1]]]|}[\filtration[\returnTimePointsEnd[i - 1]]] \leq c$, and since for all $i \in \N_{> 0}$ such that $i \leq \stoppingTimeThird$ it holds that $\returnTimePointsEnd[i - 1] < \stoppingTimeOther$, it follows for all $i \in \N_{> 0}$, by adding $\returnTimePointsEnd[i - 1] - \returnTimePointsEnd[i - 1]$ inside the expectation, that $\E{|\runTimePoints{$i$} - \E{\runTimePoints{$i$}}[\filtrationOther[i - 1]]|}[\filtrationOther[i - 1]] \cdot \indicator{i - 1 < \stoppingTimeThird} \leq c$.
    Applying \Cref{thm:walds_equation} and taking the expectations on both sides yields $\E{\sum_{i \in [\stoppingTimeThird]} \runTimePoints{$i$}}[\eventForRestart] = \E{\sum\nolimits_{i \in [\stoppingTimeThird]} \E{\runTimePoints{$i$}}[\filtrationOther[i - 1]]}[\eventForRestart]$.

    \textbf{Concluding:}
    Taking everything together and noting that for all $i \in \N_{> 0}$ such that $i \leq \stoppingTimeThird$ it holds that $\returnTimePointsStart[i - 1] < \stoppingTimeOther$ and that $\returnTimePointsEnd[i - 1] < \stoppingTimeOther$ concludes the proof.
\end{proof}

With respect to our setting, utilizing the definitions from \cref{eq:definition_of_the_return_times}, we get the following result.

\begin{corollary}[Translating results from \plateau to \asymmetricPlateau]
    \label{thm:restart_argument}
    Let $n \in \N_{> 0}$ be even, let $\plateauDistance \in [n/2]$, and let~\initialization be a distribution over $\{0, 1\}^n$.
    Consider an EA~\myAlgorithm, respresented by the sequence $(\individual[\timePoint])_{\timePoint \in \N}$ with initialization distribution~\initialization, adapted to a filtration \sequence{\filtration}[i][\N].

    Assume that when~\myAlgorithm optimizes \plateau[\plateauDistance], there is a probability $\probabilityOfImmediateSuccess \in (0, 1]$ that this solution is also an optimum of \asymmetricPlateau.
    Furthermore, assume the definitions of \cref{eq:definition_of_the_return_times}.

    Let~\stoppingTimeOther denote the run time of~\myAlgorithm on \asymmetricPlateau[\plateauDistance], and let $\stoppingTimeThird = \sup \{i \in \N \mid \returnTimePointsEnd[i] \leq \stoppingTimeOther\}$.
    Assume that $\sum_{i \in [\stoppingTimeThird]} (\returnTimePointsEnd[i] - \returnTimePointsEnd[i - 1])$ and~\stoppingTimeThird are integrable as well as that there is a $c \in \R$ such that for all $i \in \N_{> 0}$, it holds that $\E{|\returnTimePointsEnd[i] - \E{\returnTimePointsEnd[i]}[\filtration[\returnTimePointsStart[i - 1]]]|}[\filtration[\returnTimePointsStart[i - 1]]] \leq c$ and $\E{|\returnTimePointsStart[i - 1] - \E{\returnTimePointsStart[i - 1]}[\filtration[\returnTimePointsEnd[i - 1]]]|}[\filtration[\returnTimePointsEnd[i - 1]]] \leq c$.

    Last, let~\eventForRestart denote the event that $U \geq 1$, and let~\stoppingTime denote the run time of~\myAlgorithm on \plateau[\plateauDistance].
    Then
    \begin{align*}
        \E{\stoppingTimeOther} &= \E{\stoppingTime} + (1 - \probabilityOfImmediateSuccess) \mathrm{E}\Big[\sum\nolimits_{i \in [\stoppingTimeThird]} \big(\E{\returnTimePointsEnd[i] - \returnTimePointsStart[i - 1]}[\filtration[\returnTimePointsStart[i - 1]]]\\
        &\hspace*{3.1 cm}+ \E{\returnTimePointsStart[i - 1] - \returnTimePointsEnd[i - 1]}[\filtration[\returnTimePointsEnd[i - 1]]]\big)\,\Big\vert\,\eventForRestart\Big]\,,
    \end{align*}
    where we interpret the second summand to be~$0$ if $1 - p_0 = 0$.
\end{corollary}

\Cref{thm:restart_argument} states an equality for the expected run time of an EA on \asymmetricPlateau.
Although calculating the exact expectations of the right-hand side generally seems implausible, the theorem can still be used in order to derive upper and lower bounds for expected run times on \asymmetricPlateau by deriving corresponding bounds for all of the expected values in question.
Using the notation of the theorem, an easy way for getting a bound on the expectation conditional on~\eventForRestart is to derive worst-case bounds for both of the conditional expectations in the sum.
If these worst-case bounds are constants, then they can be pulled out of the expectation, and it remains to bound $\E{\stoppingTimeThird}[\eventForRestart]$, which requires to only reason about the number of times of attempting to find an optimum of \asymmetricPlateau when optimizing \plateau.
It is worth noting that the second of the conditional expectations in the sum considers the run time of~\myAlgorithm optimizing \plateau when starting with a solution with a majority of~$1$s.
Since this is typically very similar to how~\stoppingTime is defined, as only the initialization (and the history of~\myAlgorithm) differs between both scenarios, this expectation may already follow from a bound on \E{\stoppingTime}.
Overall, the complexity of the run time analysis for \asymmetricPlateau is drastically reduced.

In the following, we apply the technique above in order to derive run time bounds for \rls on \asymmetricPlateau.
Note that we already have an upper bound on the expected run time on \plateau for arbitrary initializations, due to \Cref{thm:eas_on_plateaus}.
Thus, we are only concerned with how often the algorithms find an optimum~\individual of \plateau with a majority of~$0$s before optimizing \asymmetricPlateau, and with how long it takes to get from~\individual to a solution with a majority of~$1$s.
While the latter question is specific to each algorithm, we answer the former for all unbiased algorithms.

\subsubsection{Number of Retries}
Using the notation of \Cref{thm:restart_argument}, we bound \E{\stoppingTimeThird}[D] from above for $(1 + 1)$-unary unbiased elitist algorithms with memory-restriction, as detailed in \Cref{alg:1+1_unbiased}.
Further, we require the sequence of mutation operator $(\mutation[\timePoint])_{\timePoint \in \N}$ of such an algorithm to be \emph{\onemax-compliant,} which means that each mutation operator satisfies that results of individuals with many~$1$s are at least as likely to have as many~$1$s as results of individuals with fewer~$1$s.
Formally, we say that~\mutation is \onemax-compliant if and only if, for all $\timePoint_1, \timePoint_2 \in \N$, all $\individualOther, \individualThird \in \{0, 1\}^n$ with $\individualOther = \individual[\timePoint_1]$, $\individualThird = \individual[\timePoint_2]$, and $|\individualOther|_1 \leq |\individualThird|_1$, and all $i \in [0 .. n]$, it holds that $\Pr{|\mutation[\timePoint_1](\individualOther)|_1 \geq i}[\individualOther] \leq \Pr{|\mutation[\timePoint_2](\individualThird)|_1 \geq i}[\individualThird]$.
Note that while typical mutation operators are \onemax-compliant, operations such as the inversion of a bit string are not.
This definition leads to the following result.

\begin{algorithm}[t]
    \caption{\label{alg:1+1_unbiased}
        The $(1+1)$-unary unbiased algorithm framework with unbiased and memory-restricted mutation operators $(\mutation[\timePoint])_{\timePoint \in \N}$ over~$\{0, 1\}^n$ and initial distribution~$D$, maximizing function $f\colon \{0, 1\}^n \to \R$.
        Note for all $\timePoint \in \N$ that~\mutation[\timePoint] only has access to a single individual, resulting in a heavy memory restriction.
    }
    $\individual[0] \sim D$\;
    \For{$\timePoint \in \N$}
    {
        $\individualOther \gets \mutation[\timePoint](\individual[\timePoint])$\;
        \lIf{$f(\individualOther) \geq f(\individual[\timePoint])$}
        {
            $\individual[\timePoint + 1] \gets \individualOther$%
        }
        \lElse
        {
            $\individual[\timePoint + 1] \gets \individual[\timePoint]$%
        }
    }
\end{algorithm}

\begin{lemma}
    \label{lem:number_of_restarts}
    Assume the setting of \Cref{thm:restart_argument} and that $\runTimeSet \defeq \{\returnTimePointsEnd[i + 1] - \returnTimePointsStart[i] \mid i \in \N \land \returnTimePointsEnd[i + 1] \leq \stoppingTimeOther\}$ is independent.
    If~\myAlgorithm is an instance of \Cref{alg:1+1_unbiased} with \onemax-compliant mutation~\mutation, then, conditional on~\eventForRestart, it holds that~\stoppingTimeThird is stochastically dominated by a geometrically distributed random variable with expectation~$2$.
\end{lemma}

\begin{proof}
    Note that $|\runTimeSet| = \stoppingTimeThird$, and recall that $\stoppingTimeThird \geq 1$, as we condition on~\eventForRestart.

    Let $i \in \N$ such that $\retry[i] \coloneqq \returnTimePointsEnd[i + 1] - \returnTimePointsStart[i]$ and $\returnTimePointsEnd[i + 1] \leq \stoppingTimeOther$, that is, $\retry[i] \in \runTimeSet$.
    We call such an element a \emph{retry.}
    We say that~\retry[i] is a \emph{success} if $\returnTimePointsEnd[i + 1] = \stoppingTimeOther$, and a \emph{failure} otherwise.
    Note that~\stoppingTimeThird is the number of retries until the first success and that, by assumption, the retries are mutually independent.

    It remains to show that each retry has a probability of at least~$1/2$ of being a success.
    To this end, let $i \in \N$ such that~\retry[i] is a retry.
    First, assume that $|\individual[\returnTimePointsStart[i]]|_1 = n/2$.
    Recall that the retry ends once~\myAlgorithm finds an optimum of \plateau[\plateauDistance].
    Since~\myAlgorithm is unbiased, each trajectory of~\individual that ends with a solution with a majority of~$0$s has a symmetric trajectory (in the sense of swapping~$0$s and~$1$s) with equal probability that ends with a solution with a majority of~$1$s.
    Thus, the probability of~\retry[i] being a success is~$1/2$.

    Now, assume that $|\individual[\returnTimePointsStart[i]]|_1 > n/2$.
    Let $\individualOther[0] \in \{0, 1\}^n$ such that $|\individualOther[0]|_1 = n/2$, and, for all $\timePoint \in \N$, let $\individualOther[\timePoint + 1] = \mutation[\timePoint](\individualOther[\timePoint])$.
    Since~\mutation is \onemax-compliant, it holds for all $\timePoint \in \N$ that $|\individual[\returnTimePointsStart[i] + \timePoint]|_1$ stochastically dominates $|\individualOther[\timePoint]|_1$.
    Thus, there exists a coupling such that for all $\timePoint \in \N$, it holds that $|\individual[\returnTimePointsStart[i] + \timePoint]|_1 \geq |\individualOther[\timePoint]|_1$ \cite[Theorem~$1.8.10$]{Doerr18Probabilistic}.
    Thus, the probability of~\retry[i] being a success is at least as large as the probability of the process~\individualOther finding a solution with at least $n/2 + \plateauDistance$ $1$s before finding a solution with at least that many~$0$s.
    By the first case above, the probability of this event is~$1/2$.
    This concludes the proof.
\end{proof}

We note that \rls uses \onemax-compliant mutation.

\subsubsection{Integrability of the Stopping Time}

An important  property to check in \Cref{thm:restart_argument} is that $\sum_{i \in [\stoppingTimeThird]} (\returnTimePointsEnd[i] - \returnTimePointsEnd[i - 1])$ is integrable.
To this end, it suffices to show that this random variable has \emph{some} finite expectation, as the sum is non-negative.
We do so by applying \Cref{thm:additive_drift}.

\begin{lemma}
    \label{lem:integrability_of_the_stopping_time}
    Assume the setting of \Cref{thm:restart_argument}.
    Let $(\filtration[\timePoint])_{\timePoint \in \N}$ denote the natural filtration of~\individual, and assume that there is a $p \in (0, 1]$ such that, for all $t \in \N$, it holds that
    \begin{align}
        \label{eq:integrability_of_the_stopping_time:probability_bound}
        \Pr{|\individual[\timePoint + 1]|_1 > |\individual[\timePoint]|_1}[\filtration[\timePoint]][\big] \cdot \indicator{\timePoint < \stoppingTimeOther} \geq p\! \cdot\! \indicator{\timePoint < \stoppingTimeOther} .
    \end{align}
    Then $\sum_{i \in [\stoppingTimeThird]} (\returnTimePointsEnd[i] - \returnTimePointsEnd[i - 1])$ is integrable.
\end{lemma}

\begin{proof}
    We aim to prove that \E{\stoppingTimeOther} is integrable via \Cref{thm:additive_drift}.
    Since $\sum_{i \in [\stoppingTimeThird]} (\returnTimePointsEnd[i] - \returnTimePointsEnd[i - 1]) \leq \stoppingTimeOther$ and the terms of the sum are non-negative, the claim then follows.
    Recall that we assume that $n \geq 2$.

    Let $(\randomProcess[\timePoint])_{\timePoint \in \N} = \big((n/p)^n - (n/p)^{|\individual[\timePoint]|_1}\big)$, and note that~\randomProcess is non-negative and adapted to~\filtration.
    Further note that \cref{eq:additive-drift:drift_condition} is satisfied for all $\timePoint \in \N$, since, using \cref{eq:integrability_of_the_stopping_time:probability_bound}, $n \geq 2$, and $p \leq 1$, it follows that
    \begin{align*}
        &\big(\randomProcess[\timePoint] - \E{\randomProcess[\timePoint + 1]}[\filtration[\timePoint]]\big) \cdot \indicator{\timePoint < \stoppingTime}\\
        &\quad= \Bigg(\left(\frac{n}{p}\right)^{|\individual[\timePoint + 1]|_1} - \left(\frac{n}{p}\right)^{|\individual[\timePoint]|_1}\Bigg)\\
            &\qquad\cdot \Pr{|\individual[\timePoint + 1]|_1 < |\individual[\timePoint]|_1}[\filtration[\timePoint]][\big] \cdot \indicator{\timePoint < \stoppingTimeOther}\\
            &\qquad+ \Bigg(\left(\frac{n}{p}\right)^{|\individual[\timePoint + 1]|_1} - \left(\frac{n}{p}\right)^{|\individual[\timePoint]|_1}\Bigg)\\
            &\qquad\quad\cdot \Pr{|\individual[\timePoint + 1]|_1 > |\individual[\timePoint]|_1}[\filtration[\timePoint]][\big] \cdot \indicator{\timePoint < \stoppingTimeOther}\\
        &\quad\geq \Bigg(1 - \left(\frac{n}{p}\right)^{|\individual[\timePoint]|_1}\Bigg) \cdot \indicator{\timePoint < \stoppingTimeOther}\\
            &\qquad+ \left(\frac{n}{p}\right)^{|\individual[\timePoint]|_1} \cdot \left(\frac{n}{p} - 1\right) \cdot p \cdot \indicator{\timePoint < \stoppingTimeOther}\\
        &\quad= \Bigg(1 + \left(\frac{n}{p}\right)^{|\individual[\timePoint]|_1} \cdot \left(\frac{n}{p} - 2\right)\Bigg) \cdot \indicator{\timePoint < \stoppingTimeOther}\\
        &\quad\geq \indicator{\timePoint < \stoppingTimeOther} .
    \end{align*}
    Applying \Cref{thm:additive_drift} with $\drift = 1$, noting that $\E{\randomProcess[\stoppingTimeOther]} \geq 0$, and then taking the expectation yields a finite bound for \E{\stoppingTimeOther}, concluding the proof.
\end{proof}

\subsubsection{Reaching a Majority of Ones}

The only thing left in order to apply \Cref{thm:restart_argument} is to bound the expected time it takes an algorithm to get from a solution with a majority of~$0$s to one with a majority of~$1$s.

\begin{theorem}
    \label{thm:rls_getting_to_a_majority_of_ones}
    Consider \rls optimizing \asymmetricPlateau[0], and assume that the initial individual~\individual[0] has at least $n/2 + 1$ $0$s.
    Then the expected run time of \rls is at most $n\big(1 + \ln(|\individual[0]|_0 - n/2)\big)/2$.
\end{theorem}

\begin{proof}
    We aim to apply \Cref{thm:multiplicative_drift}.
    To this end, let $(\individual[\timePoint])_{\timePoint  \in \N}$ denote the best individual of \rls after each iteration, and let $(\randomProcess[\timePoint])_{\timePoint \in \N}$ be such that for all $\timePoint \in \N$, it holds that $\randomProcess[\timePoint] = |\individual[\timePoint]|_0 - n/2$.
    Further, let~\filtration be the natural filtration of~\individual, and let $\stoppingTime = \inf\{\timePoint \mid \randomProcess[\timePoint] < 1\}$.
    Note that~\stoppingTime denotes the run time of \rls on \majority.

    Since $\randomProcess[0] \geq 1$ by assumption, we are left to check \cref{eq:multiplicative-drift:drift_condition}.
    Let $\timePoint \in \N$, and assume that $\timePoint < \stoppingTime$.
    Note that then
    \begin{align*}
        \randomProcess[\timePoint] - \E{\randomProcess[\timePoint + 1]}[\filtration[\timePoint]]
            &= -1 \cdot \frac{n - |\individual[\timePoint]|_0}{n} + 1 \cdot \frac{|\individual[\timePoint]|_0}{n}\\
            &= \frac{2|\individual[\timePoint]|_0 - n}{n}
            = \frac{2}{n} \randomProcess[\timePoint] .
    \end{align*}
    Applying \Cref{thm:multiplicative_drift} with $\drift = 2/n$ concludes the proof.
\end{proof}

\subsubsection{Deriving the Run Time Result}

By combining the results from the previous sections, the run time bound on \asymmetricPlateau follows straightforwardly.

\begin{proof}[Proof of \Cref{thm:rls_run_time_on_asymmetric_plateau}]
    We aim to apply \Cref{thm:restart_argument}.
    To this end, we check the assumptions of the theorem and use its notation.

    Due to the uniform initialization and the symmetry of the mutation of \rls (with respect to the number of~$0$s and~$1$s), it holds that $\probabilityOfImmediateSuccess = 1/2$.
    By \Cref{lem:integrability_of_the_stopping_time}, $\sum_{i \in [\stoppingTimeThird]} (\returnTimePointsEnd[i] - \returnTimePointsEnd[i - 1])$ is integrable, and by \Cref{lem:number_of_restarts}, using that $\E{\stoppingTimeThird} \leq \E{\stoppingTimeThird}[\eventForRestart] + \E{\stoppingTimeThird}[\overline{\eventForRestart}][\big]$ and, by definition of~\eventForRestart, that $\E{\stoppingTimeThird}[\overline{\eventForRestart}][\big] = 0$, it follows that~\stoppingTimeThird is integrable.
    The last condition to check is that there is a $c \in \R$ such that for all $i \in \N_{> 0}$, it holds that $\E{|\returnTimePointsEnd[i] - \E{\returnTimePointsEnd[i]}[\filtration[\returnTimePointsStart[i - 1]]]|}[\filtration[\returnTimePointsStart[i - 1]]] \leq c$ and $\E{|\returnTimePointsStart[i - 1] - \E{\returnTimePointsStart[i - 1]}[\filtration[\returnTimePointsEnd[i - 1]]]|}[\filtration[\returnTimePointsEnd[i - 1]]] \leq c$.
    Let $i \in \N_{> 0}$.
    For the first inequality, we note that, by the triangle inequality and Jensen's inequality,
    \begin{align*}
        &\E{|\returnTimePointsEnd[i] - \E{\returnTimePointsEnd[i]}[\filtration[\returnTimePointsStart[i - 1]]]|}[\filtration[\returnTimePointsStart[i - 1]]]\\
        &= \E{|\returnTimePointsEnd[i] - \returnTimePointsStart[i - 1] - \E{\returnTimePointsEnd[i] - \returnTimePointsStart[i - 1]}[\filtration[\returnTimePointsStart[i - 1]]]|}[\filtration[\returnTimePointsStart[i - 1]]]\\
        &\leq 2 \E{|\returnTimePointsEnd[i] - \returnTimePointsStart[i - 1]|}[\filtration[\returnTimePointsStart[i - 1]]] .
    \end{align*}
    By \Cref{thm:eas_on_plateaus}, using its notation, it holds that
    \begin{align*}
        \E{|\returnTimePointsEnd[i] - \returnTimePointsStart[i - 1]|}[\filtration[\returnTimePointsStart[i - 1]]] \leq 3\plateauDistance \cdot h\big(\individual[\returnTimePointsStart[i - 1]]\big) / (\basisRLS - 1) .
    \end{align*}
    Since, for all $\individual \in \{0, 1\}^n$, it holds that $h(\individual) \leq \basisRLS^\plateauDistance - 1$, it follows overall that $\E{|\returnTimePointsEnd[i] - \E{\returnTimePointsEnd[i]}[\filtration[\returnTimePointsStart[i - 1]]]|}[\filtration[\returnTimePointsStart[i - 1]]] \leq 6 \plateauDistance \cdot (\basisRLS^\plateauDistance - 1) / (\basisRLS - 1) \eqdef c_1$.
    Similarly, we see that
    \begin{align*}
        &\E{|\returnTimePointsStart[i - 1] - \E{\returnTimePointsStart[i - 1]}[\filtration[\returnTimePointsEnd[i - 1]]]|}[\filtration[\returnTimePointsEnd[i - 1]]]\\
        &= \E{|\returnTimePointsStart[i - 1] - \returnTimePointsEnd[i - 1] - \E{\returnTimePointsStart[i - 1] - \returnTimePointsEnd[i - 1]}[\filtration[\returnTimePointsEnd[i - 1]]]|}[\filtration[\returnTimePointsEnd[i - 1]]]\\
        &\leq 2 \E{|\returnTimePointsStart[i - 1] - \returnTimePointsEnd[i - 1]|}[\filtration[\returnTimePointsEnd[i - 1]]] .
    \end{align*}
    By \Cref{thm:rls_getting_to_a_majority_of_ones}, it holds that
    \begin{align*}
        \E{|\returnTimePointsStart[i - 1] - \returnTimePointsEnd[i - 1]|}[\filtration[\returnTimePointsEnd[i - 1]]] \leq \tfrac{n}{2}\big(1 + \ln(|\individual[\returnTimePointsEnd[i - 1]]|_0 - n/2)\big) .
    \end{align*}
    Noting that $|\individual[\returnTimePointsEnd[i - 1]]|_0 = n/2 + \plateauDistance$, since \rls cannot get to a solution with strictly more than $n/2 + \plateauDistance$ $0$s before getting exactly $n/2 + \plateauDistance$ $0$s, we get $\E{|\returnTimePointsStart[i - 1] - \E{\returnTimePointsStart[i - 1]}[\filtration[\returnTimePointsEnd[i - 1]]]|}[\filtration[\returnTimePointsEnd[i - 1]]] \leq n\big(1 + \ln(\plateauDistance)\big) \eqdef c_2$.
    Thus, choosing $c = \max\{c_1, c_2\}$ satisfies the remaining condition of \Cref{thm:restart_argument}.

    We continue with bounding \E{\stoppingTime}, noting that for all $i \in \N_{> 0}$ such that $i \in [\stoppingTimeThird]$, we bounded the expectations $\E{\returnTimePointsEnd[i] - \returnTimePointsStart[i - 1]}[\filtration[\returnTimePointsStart[i - 1]]]$ and $\E{\returnTimePointsStart[i - 1] - \returnTimePointsEnd[i - 1]}[\filtration[\returnTimePointsEnd[i - 1]]]$ already above, as the differences are non-negative and the absolute value thus does not change anything.

    Let $i \in \N_{> 0}$ such that $i \in [\stoppingTimeThird]$.
    Note that $\returnTimePointsEnd[i] - \returnTimePointsStart[i - 1]$ and~\stoppingTime have the same distribution conditional on the same initial individual for the respective time interval, since \rls is Markovian and since both stopping times stop once an optimum of \plateau[\plateauDistance] is found.
    Thus, their conditional expectations (on the same initial individual) are the same, and \Cref{thm:eas_on_plateaus} yields a bound for \E{\stoppingTime}, too.

    Combining all bounds, by the tower rule for expectation, \Cref{thm:restart_argument} yields
    \begin{align*}
        \E{\stoppingTimeOther} &\leq 3\plateauDistance \cdot \frac{\basisRLS^\plateauDistance - 1}{\basisRLS - 1}\\
        &\quad + \frac{1}{2}\E{\sum\nolimits_{i \in [\stoppingTimeThird]}\left(n\frac{1 + \ln(\plateauDistance)}{2} + 3\plateauDistance \cdot \frac{\basisRLS^\plateauDistance - 1}{\basisRLS - 1}\right)}[\eventForRestart]\\
        &\hspace*{-2.5 ex}= 3\plateauDistance \cdot \frac{\basisRLS^\plateauDistance - 1}{\basisRLS - 1} + \frac{1}{2}\left(n\frac{1 + \ln(\plateauDistance)}{2} + 3\plateauDistance \cdot \frac{\basisRLS^\plateauDistance - 1}{\basisRLS - 1}\right)\E{\stoppingTimeThird}[\eventForRestart] .
    \end{align*}
    Applying \Cref{lem:number_of_restarts} yields $\E{\stoppingTimeThird}[\eventForRestart] \leq 2$, thus concluding the proof.
\end{proof}

\subsection{Results for Neutrality in the W-Model}
\label{sec:w-model}

For all $n \in \N_{> 0}$, all $\individual \in \{0, 1\}^n$, and all $S \subseteq [n]$, let $\individual_S$ denote the bit string of length~$|S|$ that consists only of those entries in~\individual with an index in~$S$, that is, $\individual_S = (\individual_i)_{i \in S}$.

Let $n \in \N_{> 0}$.
Given a pseudo-Boolean function $f\colon \{0, 1\}^n \to \R$ as well as $k \in \N_{> 0}$, the property of neutrality in the W-model~\cite{WeiseCLW20WModel} constructs a new function $f'\colon \{0, 1\}^{nk} \to \R$ with
\begin{align*}
    \individual \mapsto f\big(\asymmetricPlateau[1](\individual_{[(i - 1)k + 1 .. ik]})_{i \in [n]}\big) .
\end{align*}
In other words, each bit of~$f$ is exchanged for a block of~$k$ bits.
When evaluating~$f'$, for each block, the majority of bits is determined, and then~$f$ is evaluated on the string of the~$n$ majority bits.
We refer to~$f'$ as \emph{$d$ with added neutrality of degree~$k$.}

Applying the results by Doerr et~al.~\cite{DoerrSW13BlockAnalysis}, we get the following bound for the expected run time of \rls on the pseudo-Boolean function $\onemax\colon \individual \mapsto |\individual|_1$ with added neutrality.

\begin{corollary}
    \label{cor:rls_w-model}
    Let $n, k \in \N_{> 0}$ such that $n$ is even.
    For sufficiently large~$n$, the expected run time of \rls on \onemax with added neutrality of degree~$k$ is at most \bigO{n k \log(n)}.
\end{corollary}

Before we state the result by Doerr et~al.~\cite{DoerrSW13BlockAnalysis} and prove \Cref{cor:rls_w-model}, we introduce some notation of their paper to the degree that it is necessary for our result.
Let $n, k \in \N_{> 0}$ and $i \in [n]$, we say that a function $g\colon \{0, 1\}^{nk} \to \R$ \emph{acts only on block~$i$} if and only if for all $\individual \in \{0, 1\}^{nk}$, it holds that $g(\individual) = g(\individual_{[(i - 1)k + 1 .. ik]})$.
We say that a function $f\colon \{0, 1\}^{nk}$ is a \emph{consecutively separable function composed of~$n$ sub-functions} if and only if there exists a collection $(g_i)_{i \in [n]}$ of~$n$ pseudo-Boolean functions, each of dimension~$nk$, such that for all $\individual \in \{0, 1\}^{nk}$, it holds that $f(\individual) = \sum_{i \in [n]} g_i(\individual)$, and for all $i \in [n]$, it holds that~$g_i$ acts only on block~$i$.

\begin{theorem}[Simplification of {\cite[Theorem~$2$]{DoerrSW13BlockAnalysis}}]
    \label{thm:sub-function_combination_run_time}
    Let $n, k \in \N_{> 0}$, and let $f\colon \{0, 1\}^{nk} \to \R$ be a consecutively separable function composed of~$n$ sub-functions.
    Further, let~$T^*$ denote the maximum expected run time of \rls with uniform initialization over all sub-functions.
    Then the expected run time of \rls with uniform initialization on~$f$ is \bigO{T^* \log(n)}[\big].
\end{theorem}

\begin{proof}[Proof of \Cref{cor:rls_w-model}]
    Note that \onemax with added neutrality of degree~$k$ is a consecutively separable function composed of~$n$ sub-functions, each of which acts only on a single block of length~$k$ as \asymmetricPlateau[1].
    Thus, we aim to apply \Cref{thm:sub-function_combination_run_time} and determine~$T^*$ by \Cref{thm:rls_run_time_on_asymmetric_plateau}, noting that each block has the same sub-function.
    However, note that \Cref{thm:rls_run_time_on_asymmetric_plateau} is not directly applicable, since it considers \asymmetricPlateau[1] to be defined over the entire search space (that is, for all bits), whereas in the setting of \Cref{cor:rls_w-model}, \asymmetricPlateau[1] is a sub-function that only acts on a block of length~$k$.
    Hence, we translate the result from \Cref{thm:rls_run_time_on_asymmetric_plateau} to our setting.
    We do so by only counting those mutations that change a bit in block~$i$ and adding waiting time due to iterations that change other (irrelevant) bits for a specific sub-function.

    Formally, let $i \in [n]$, and consider the \asymmetricPlateau[1] sub-function~$g_i$ acting only on block~$i$.
    Further, let~\stoppingTimeOther denote the expected run time of \rls with uniform initialization on \asymmetricPlateau[1] defined over bit strings of length~$k$, and let~\stoppingTime denote the expected run time of \rls with uniform initialization on~$g_i$.
    We note that, similar to the proof of \Cref{lem:integrability_of_the_stopping_time}, \stoppingTime is integrable, allowing us to reorder terms when calculating \E{\stoppingTime}.
    Last, let $(\individual[\timePoint])_{\timePoint \in \N}$ denote the trajectory of \rls with uniform initialization on~$g_i$, and for all $\timePoint \in \N_{> 0}$, let~$M_\timePoint$ denote the event that~\individual[\timePoint] is a result of a mutation that changes a bit in block~$i$.
    We couple~\stoppingTimeOther and~\stoppingTime such that
    \begin{align}
        \label{eq:coupling}
        \sum\nolimits_{\timePoint \in [\stoppingTime]} \indicator{M_\timePoint} = \stoppingTimeOther\ ,
    \end{align}
    that is, only considering changes in block~$i$, the optimization behaves as if considering the optimization of \asymmetricPlateau[1] defined of bit strings of length~$k$.
    Further, let $I = \{\timePoint \in [\stoppingTime] \mid \indicator{M_\timePoint} = 1\}$.

    Since it trivially holds that $\stoppingTime = \sum_{\timePoint \in [\stoppingTime]} (\indicator{M_\timePoint} + \indicator{\overline{M_\timePoint}})$, by the definition of~$I$, it follows that $\stoppingTime = \sum_{\timePoint \in I} \indicator{M_\timePoint} + \sum_{[\stoppingTime] \smallsetminus I} \indicator{\overline{M_\timePoint}}$.
    We aim to connect the second sum to the first, since we know how to bound the first sum by \Cref{eq:coupling}.
    To this end, for all $\timePoint \in I$, let $J_t = \{\timePoint' \in [\stoppingTime] \mid \timePoint' \leq \timePoint \land \forall s \in I, s < \timePoint\colon s < \timePoint'\}$, that is, the set of all time points that are in between the time points in~$I$.
    Note that $\indicator{M_\stoppingTime} = 1$ and, thus, $\stoppingTime \in I$, since the final mutation has to change a bit in block~$i$.
    Thus, $[\stoppingTime] = \bigcup_{\timePoint \in I} J_\timePoint$.
    Further note for all $\timePoint \in I$ that $|J_\timePoint|$ follows a geometric distribution with support~$\N_{> 0}$ and with success probability $k/(nk) = 1/n$, since the mutation of~\rls chooses uniformly at random which bit to flip.
    In addition, the $\{|J_\timePoint|\}_{\timePoint \in I}$ are independently and identically distributed.
    Let~$J$ denote their common distribution.

    Using the definitions of $\{|J_\timePoint|\}_{\timePoint \in I}$ and~$J$ as well as \cref{eq:coupling}, we see that $\stoppingTime = \sum_{\timePoint \in I} |J_\timePoint| \cdot \indicator{M_\timePoint} = J \cdot \stoppingTimeOther$.
    Since~$J$ and~\stoppingTimeOther are independent, as the mutation of~\rls is independent of any random choice besides which bit it flips, it follows that $\E{J \cdot \stoppingTimeOther} = \E{J} \cdot \E{\stoppingTimeOther}$.
    By definition, it follows that $\E{J} = n$, and by \Cref{thm:rls_run_time_on_asymmetric_plateau}, it follows that $\E{\stoppingTimeOther} \leq 6 + k/2$.
    Thus, $\E{\stoppingTime} = \bigO{nk}$.

    Overall, if follows that $T^* = \bigO{nk}$.
    Applying \Cref{thm:sub-function_combination_run_time} concludes the proof.
\end{proof}

\section{Empirical Results for \texorpdfstring{\rls[\ell]}{RLS\_𝓁}}
\label{sec:experiments}
\begin{figure*}[t]
    \begin{subfigure}[b]{0.3\textwidth}
        \includegraphics[width = \textwidth]{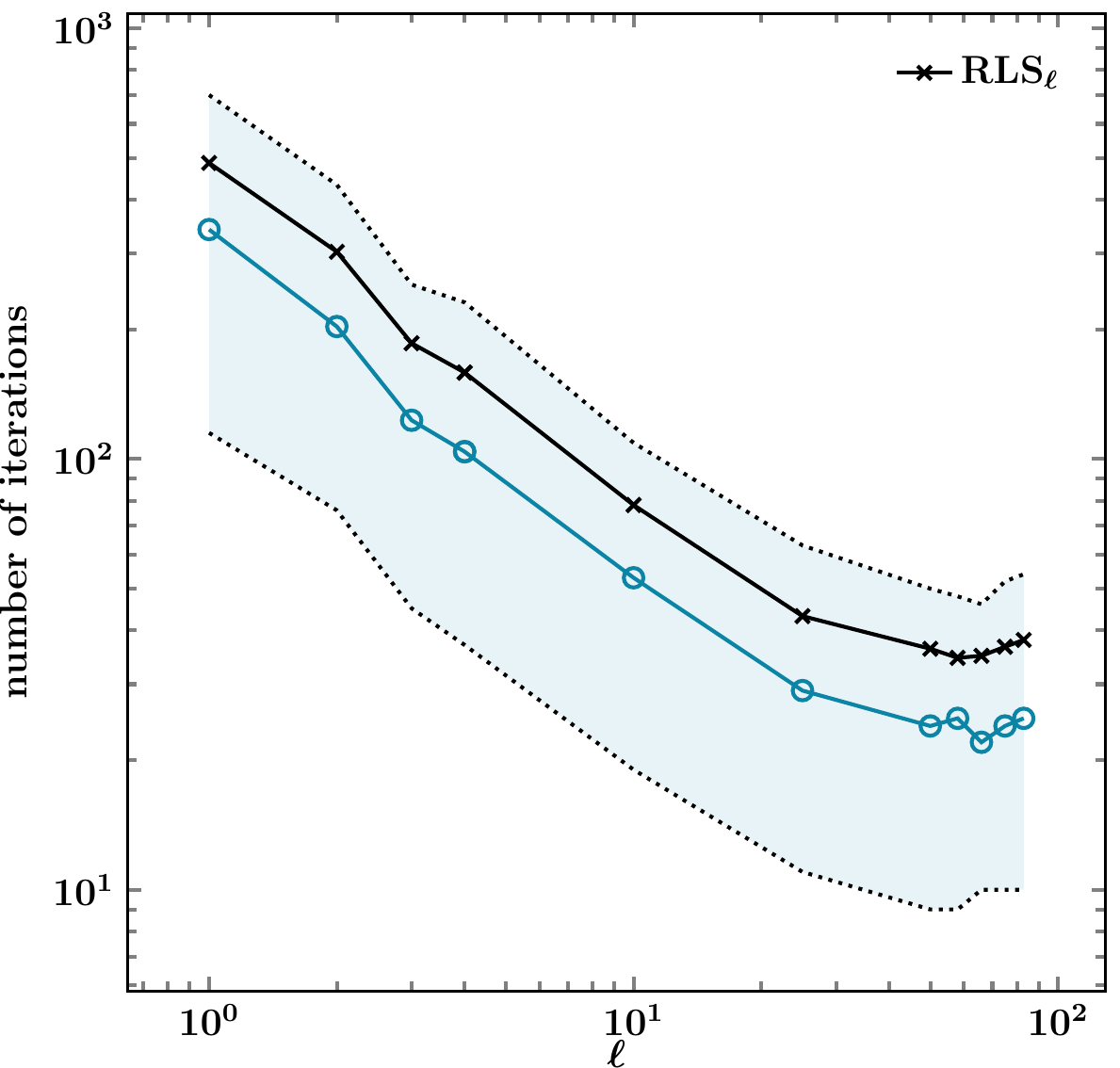}
        \caption{$n = 100$}
    \end{subfigure}
    \hfil
    \begin{subfigure}[b]{0.3\textwidth}
        \includegraphics[width = \textwidth]{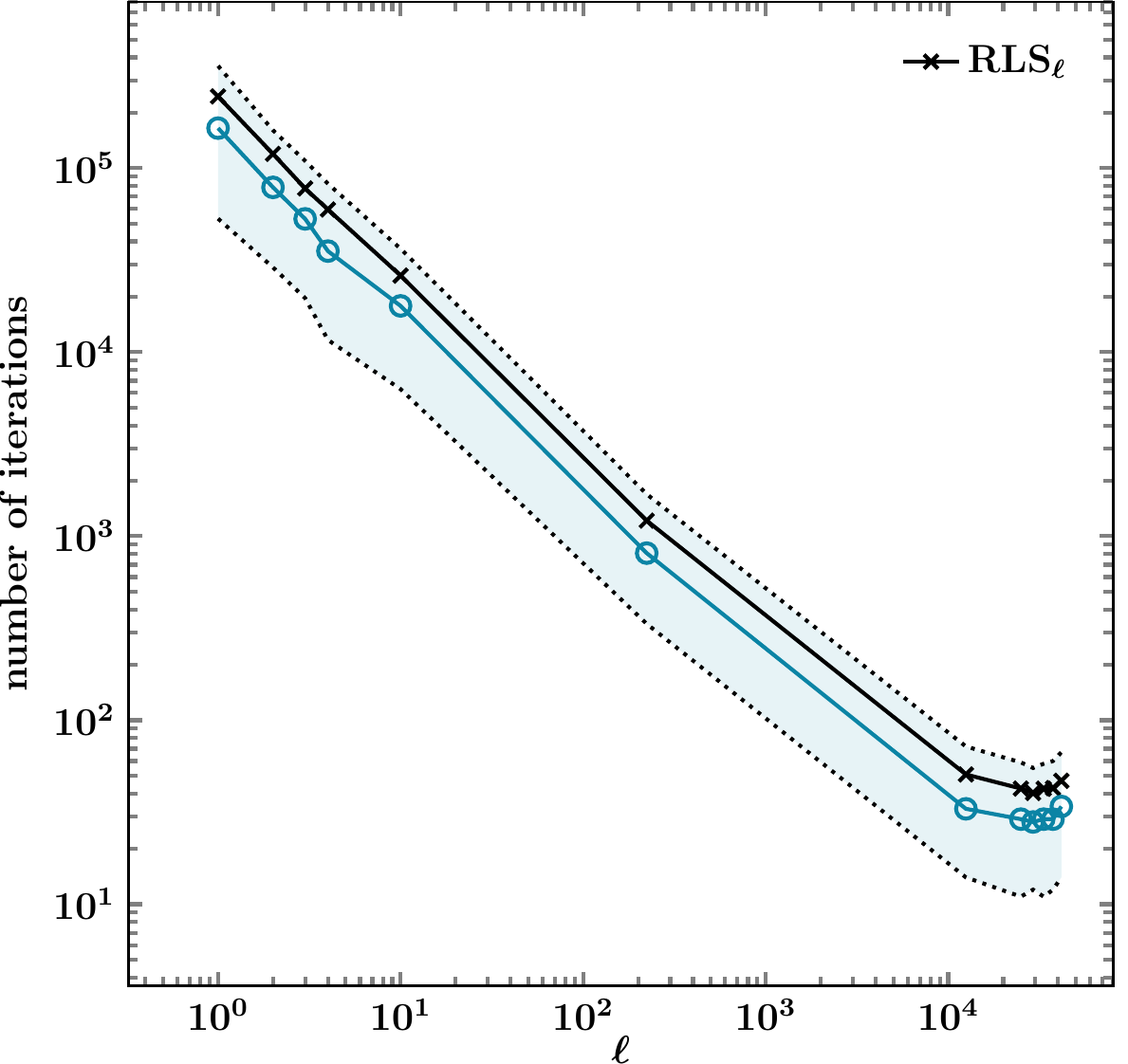}
        \caption{$n = 50\,000$}
    \end{subfigure}
    \hfil
    \begin{subfigure}[b]{0.3\textwidth}
        \includegraphics[width = \textwidth]{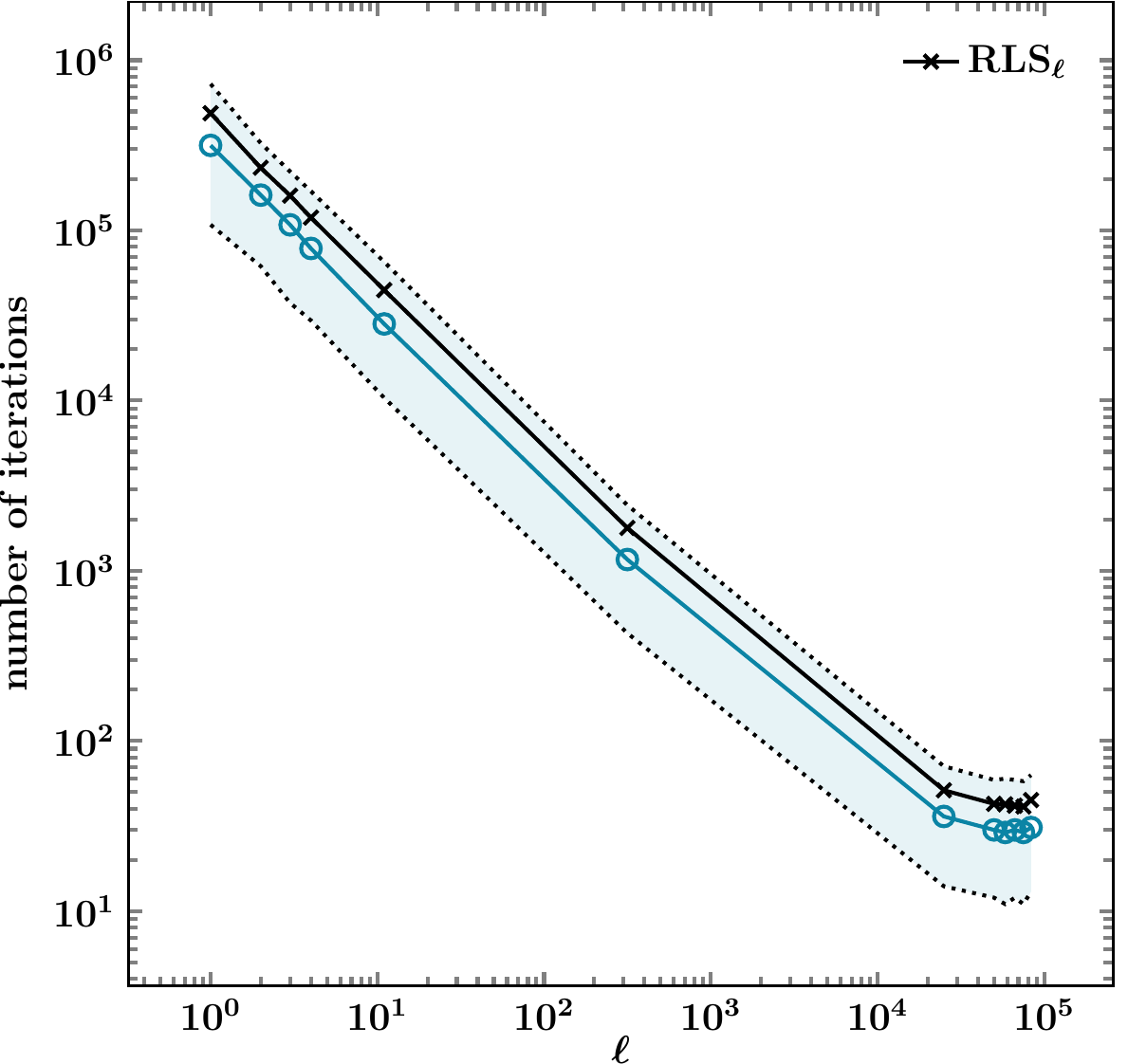}
        \caption{$n = 100\,000$}
    \end{subfigure}
    \caption{
        \label{fig:empirical_run_time_varying_n_and_k}
        The empirical run time (number of iterations until an optimum was found) of \rls[\ell] on \asymmetricPlateau[\lfloor \sqrt{n} \rfloor] for $n \in \{100, 50\,000, 100\,000\}$, each with varying values of~$\ell$, as discussed at the beginning of \Cref{sec:empirical_run_time_varying_n_and_k}.
        Note the doubly logarithmic scale.
        For each combination of~$n$ and~$\ell$, $1000$ independent runs with uniform initialization were started.
        The solid black line (points represented by~$\times$) depicts the average of the runs.
        The solid blue line (points represented by~$\circ$) depicts the median of the runs, and the surrounding shaded area depicts the range from the $25$th percentile to the $75$th percentile.
    }
\end{figure*}

We complement our theoretical results on \asymmetricPlateau with empirical investigations on a generalization of \rls called \rls[\ell], which flips exactly~$\ell$ bits each iteration.
In more detail, \rls[\ell] follows \Cref{alg:1+1_unbiased} and uses a time-homogeneous mutation operator, that is, it uses the same mutation operator in each iteration.
Let $n \in \N_{> 0}$, and let $\individual \in \{0, 1\}^n$.
Given an $\ell \in [n]$, for all $\timePoint \in \N$, \rls[\ell] computes the result of $\mutation[\timePoint](\individual)$ by first choosing a subset $I \subseteq [n]$ of cardinality~$\ell$ uniformly at random among all $\ell$-size subsets of $[n]$ and then flipping exactly the bits at the positions in~$I$.
Note that \rls is a special case of \rls[\ell] for $\ell = 1$.

In our experiments, we analyze the impact of the parameter~$\ell$ of \rls[\ell] on the run time for \asymmetricPlateau when using the uniform distribution as initialization distribution.
To this end, given a problem size of $n \in \N_{> 0}$, with even~$n$, we fix the parameter~\plateauDistance of \asymmetricPlateau[\plateauDistance] to~$\lfloor \sqrt{n} \rfloor$.
Recalling our discussion of~\plateauDistance after \Cref{thm:rls_run_time_on_asymmetric_plateau}, this results in an expected run time linear in~$n$ for \plateau[\plateauDistance], which is a lower bound for the run time for \asymmetricPlateau (see \Cref{thm:restart_argument}).
For larger values of~\plateauDistance, the expected run time increases drastically.
Thus, $\plateauDistance = \lfloor \sqrt{n} \rfloor$ results in an easily-solvable problem that is not yet trivial.
We note though that, due to the uniform initialization, there is a constant probability of the initial solution being already optimal.

\subsection{Near-optimal Value of \texorpdfstring{$\ell$}{l}}
\label{sec:empirical_run_time_varying_n_and_k}
We determine the value~$\ell$ of \rls[\ell] that has the best average empirical run time among a variety of different values of~$\ell$.
For each each problem size $n \in \{10^i \mid i \in [2 .. 5]\} \cup \{5 \cdot 10^i \mid i \in [2 .. 4]\}$, we consider $\ell \in \{1, 2, 3, 4, \lfloor \ln n \rfloor, \lfloor \sqrt{n} \rfloor\} \cup \{\lfloor an/12 \rfloor \mid a \in \{3\} \cup [6 .. 10]\}$.
For each parameter combination of~$n$ and~$\ell$, we start $1000$ independent runs of \rls[\ell] (on \asymmetricPlateau[\lfloor \sqrt{n} \rfloor]) and store its empirical run time, that is, the number of iterations~\timePoint from \Cref{alg:1+1_unbiased} until an optimal solution was found for the first time.
Afterward, we compute the average, median, and the $25$th as well as $75$th percentile.
\Cref{fig:empirical_run_time_varying_n_and_k} shows the results for $n \in \{100, 50\,000, 100\,000\}$.

The results for our values of~$n$ look qualitatively similarly, that is, the run time drastically decreases by several orders of magnitude with increasing~$\ell$ until $\ell \in \bigTheta{n}$.
The median is always below the mean, and the area between the $25$th percentile and the mean is larger than the one between the mean and the $75$th percentile.
This is likely due to the initialization creating an optimal solution or one that is close to an optimal one with respect to its number of~$1$s.
Such runs result in a low run time.
The mean being larger than the median suggests that initial solutions that are further away from optimal solutions in terms of their number of~$1$s result in a far larger run time.

In the regime of $\ell \in \{\lfloor an/12 \rfloor \mid a \in \{3\} \cup [6 .. 10]\}$, the minimum average empirical run time is always taken for one of the values $\ell \in \{\lfloor an/12 \rfloor \mid a \in [6 .. 9]\}$.
For these values of~$\ell$, the difference of the smallest average to the second smallest is typically less than~$1$.
This suggests that, in this regime, the run time of \rls[\ell] does not dependent heavily on~$\ell$.
However, overall, a linear dependence of~$\ell$ on~$n$ seems best.

\subsection{Trajectories for Different Values of \texorpdfstring{$\ell$}{l}}
The results in \Cref{fig:empirical_run_time_varying_n_and_k} show that the value of~$\ell$ has a drastic impact on the run time of \rls[\ell] on \asymmetricPlateau[\lfloor \sqrt{n} \rfloor].
Since~$\ell$ only influences how largely the mutation changes the current solution but does not influence the selection of which solution to use for the next iteration, we compare the trajectories of \rls[\ell] for different values of~$\ell$.
We choose $n = 10\,000$ and $\ell \in \{1, \lfloor \ln n\rfloor, \lfloor n/2 \rfloor\}$, and we log at each iteration the number of~$1$s of the currently best solution (the \emph{fitness level}).
Note that all optimal solutions are at fitness levels at least $5100$.

\begin{figure*}[t]
    \begin{subfigure}[b]{0.3\textwidth}
        \includegraphics[width = \textwidth]{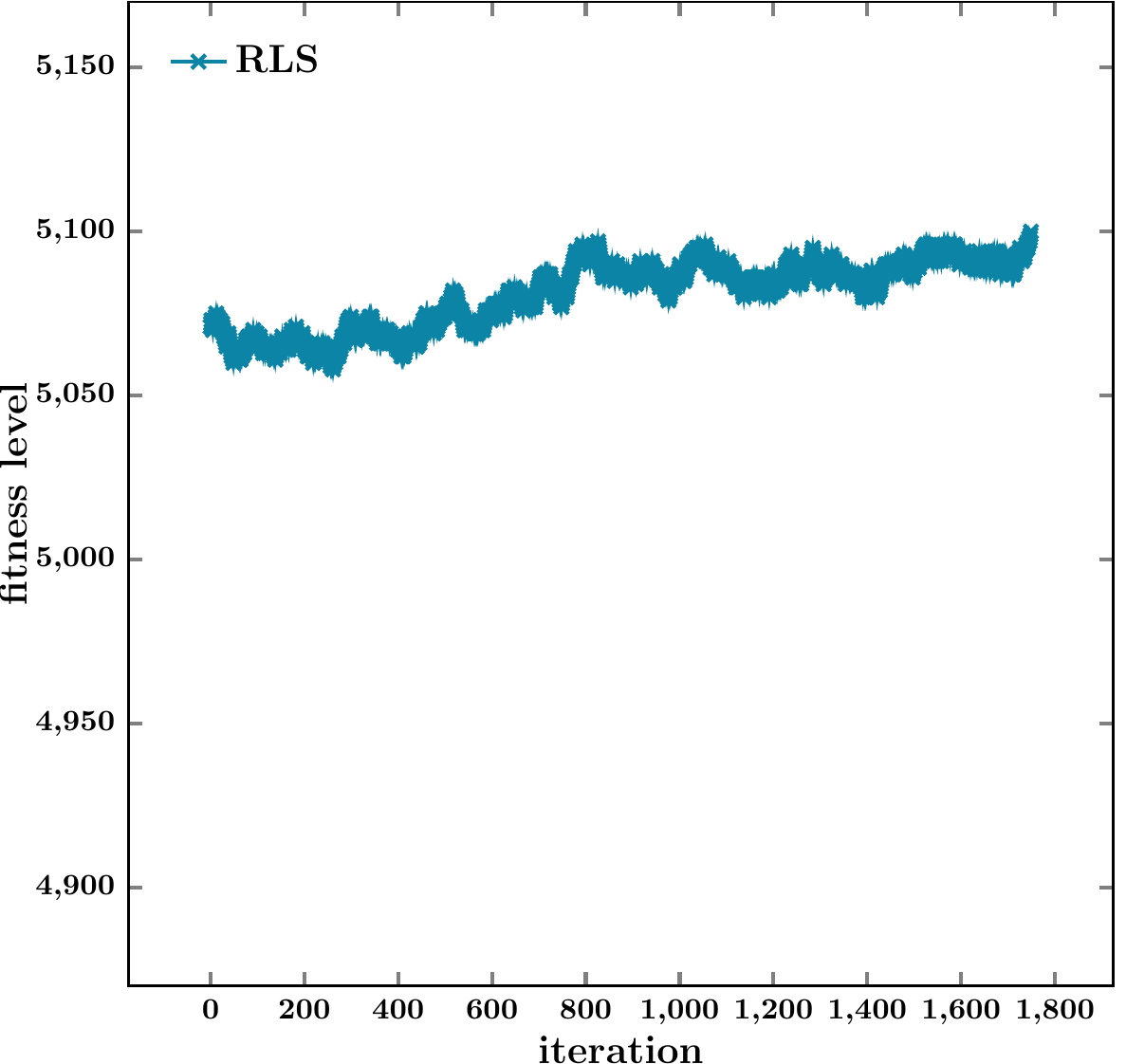}
        \caption{$\ell = 1$}
    \end{subfigure}
    \hfil
    \begin{subfigure}[b]{0.3\textwidth}
        \includegraphics[width = \textwidth]{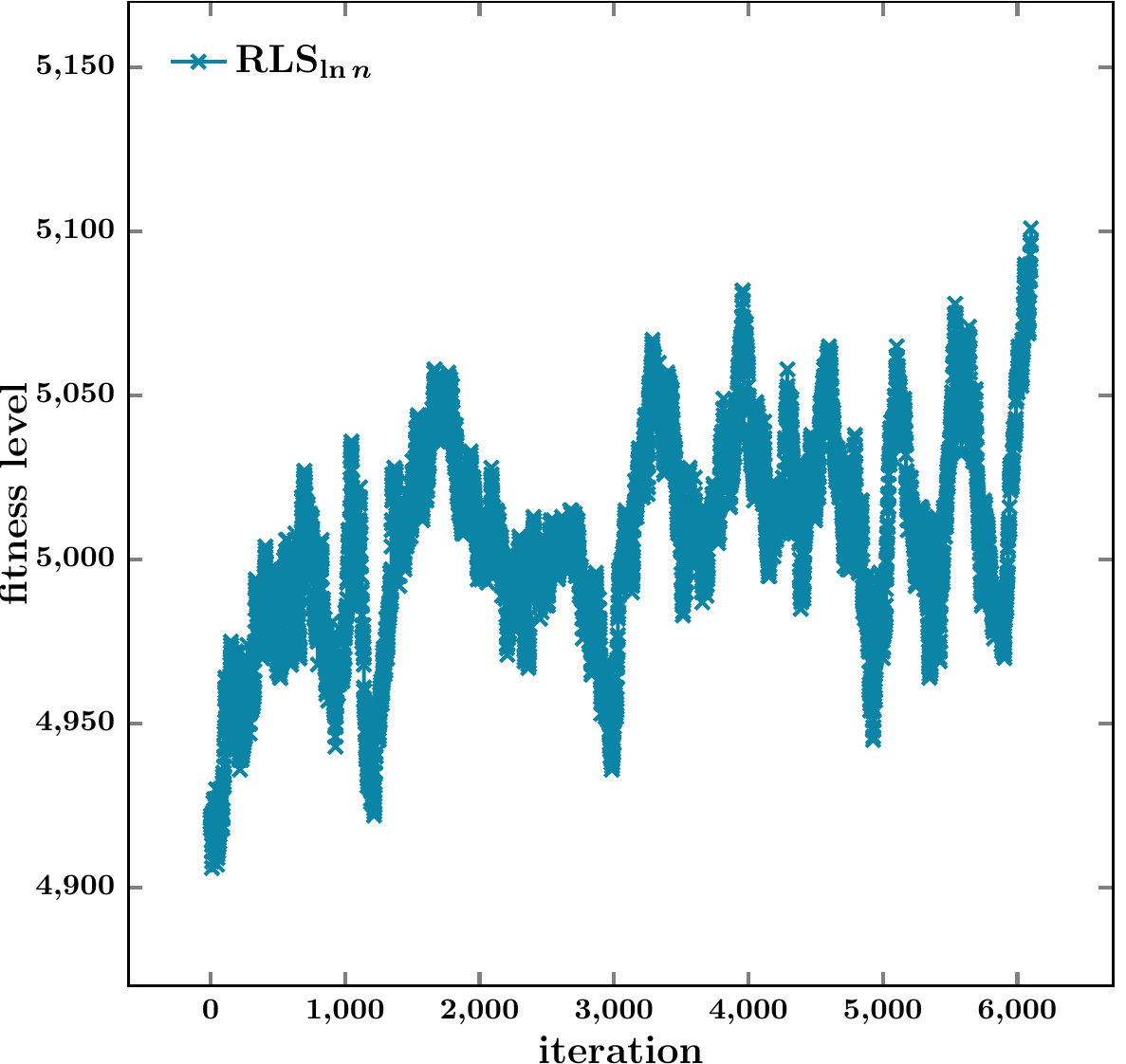}
        \caption{$\ell = \lfloor \ln n\rfloor$}
    \end{subfigure}
    \hfil
    \begin{subfigure}[b]{0.3\textwidth}
        \includegraphics[width = \textwidth]{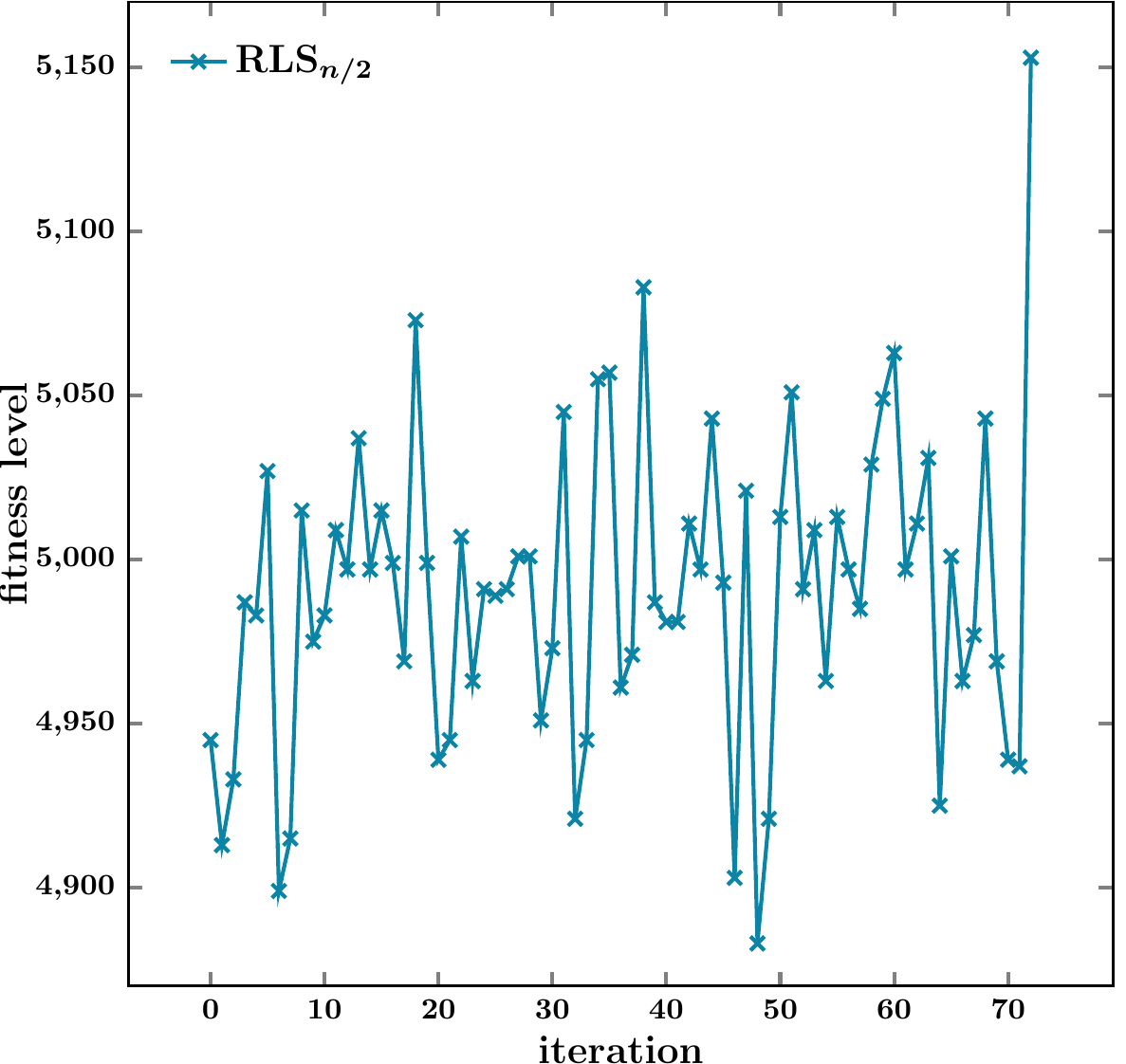}
        \caption{$\ell = \lfloor n/2 \rfloor$}
    \end{subfigure}
    \caption{
        \label{fig:different_trajectories}
        The number of~$1$s \emph{(fitness level)} of the best solution of \rls[\ell] on \asymmetricPlateau[\lfloor \sqrt{n} \rfloor] for $n = 10\,000$ and for different values of~$\ell$.
        Each plot depicts a single run that was stopped once an optimal solution was found (fitness level of at least $5100$).
        Note that we show runs that started from a non-optimal solution, chosen uniformly at random among all possible $2^n$ solutions.
    }
\end{figure*}

\Cref{fig:different_trajectories} depicts our results.
Note that there is a huge difference in run time between $\ell = 1$ and $\ell = \lfloor \ln n\rfloor$, which is most likely due to the former starting with a solution of more than $5000 = n/2$ $1$s, whereas the latter starts at a position almost at $n/2 - \sqrt{n}$.
This highlights the impact of the initialization on the run time.
Interestingly, the case $\ell = \lfloor n/2 \rfloor$ also starts with a rather bad solution of about $n/2 - \sqrt{n}/2$ $1$s but manages to quickly find an optimum.
In fact, the last mutation increases the number of~$1$s by about~$200$.
This suggests that \rls[\ell] has far better capabilities of exploring the plateau with $\ell \in \bigTheta{n}$ than with smaller values of~$\ell$.

\section{Conclusion}
\label{sec:conclusion}
We analyzed the expected run time that RLS, started on a plateau, requires until it leaves the plateau for the first time.
We considered a symmetric (\plateau) and an asymmetric (\asymmetricPlateau) setting, and we showed, for general EAs, how to derive a result for \asymmetricPlateau from a result on \plateau (\Cref{thm:restart_argument}), based on a more general result for a random process to reach a certain offset (\Cref{thm:general_restart_argument}).
A fair amount of conditions of \Cref{thm:restart_argument} follow directly from \Cref{lem:number_of_restarts,lem:integrability_of_the_stopping_time}, which apply to many EAs.
\Cref{lem:integrability_of_the_stopping_time} is so general that it holds for any EA that has, for each individual, a positive probability generating it.
And although \Cref{lem:number_of_restarts} is not as easy to check, we note that it holds for the \onePlusOneEA~\cite[Lemma~$6.1$]{Witt13EAsOnLinearFunctions}.

\Cref{thm:general_restart_argument} provides a tool for analyzing \emph{general} plateaus by considering phases in which the algorithm tries to cross the plateau.
In order to bound the expected number of restarts, one determines the probability of crossing the plateau and not returning to its start.
This can be done by defining a drift potential that decreases toward the end of the plateau (as done in the proof of \Cref{thm:eas_on_plateaus}) but is also~$0$ at the start.
The optional-stopping theorem allows then to determine the desired probability.
However, we note that simply using the potential from the proof of \Cref{thm:eas_on_plateaus} is typically a too coarse estimate since the bound always assumes the worst case of crossing the plateau.
It remains an interesting open problem to derive tighter bounds for such a setting.

Further possible future work includes deriving a lower bound for \rls on \plateau, as well as analyzing the expected run time of more algorithms on \plateau, such as the \onePlusOneEA, \rls[\ell], or non-elitist EAs.
Any such bound translates almost directly into a bound for \asymmetricPlateau.
Similarly, combining our results with those of Doerr et~al.~\cite{DoerrSW13BlockAnalysis} on separable functions provides good bounds for \rls and the \onePlusOneEA on \onemax functions with any degree of neutrality from the W-model.
Extending the results of Doerr et~al.~\cite{DoerrSW13BlockAnalysis} to other algorithms or to functions other than \onemax also helps extend the picture about how well EAs cope with neutrality.
Last, the current definitions of \plateau and \asymmetricPlateau let the algorithm either start on a plateau or in a global optimum.
A potential generalization is to place the plateau somewhere else in the search space and introduce an easy slope toward the plateau.
This follows the same idea as the function \textsc{Plateau} by Antipov and Doerr~\cite{AntipovD18PrecisePlateauAnalysis} but generalizing it even further such that the function has more than a single optimum.

Our long-term objective is to analyze the impact of the other components of the W-model problem generator proposed by Weise et~al.~\cite{WeiseCLW20WModel} and extended by Doerr et~al.~\cite{DoerrYHWSB20IOHprofiler}---first individually for each module (as done here for \emph{neutrality}) and then for combinations of the four layers (\emph{neutrality}, \emph{dummy variables}, \emph{epistasis}, and \emph{ruggedness}). We see this as an important step towards run time results that more explicitly link algorithms' performance to problem characteristics.

\bibliographystyle{IEEEtran}
\balance
\bibliography{references}

\begin{IEEEbiographynophoto}
    {Carola Doerr,} formerly Winzen, is since 2013 a permanent researcher at Sorbonne Unversité in Paris, France.
    Carola obtained her PhD in Computer Science from Saarland University and the Max Planck Institute for Informatics in 2011, and she successfully defender her habilitation (HDR) at Sorbonne Université in 2020.
    She works on theoretical analysis, benchmarking, and practical applications of black-box optimization heuristics.
\end{IEEEbiographynophoto}

\begin{IEEEbiographynophoto}
    {Martin~S. Krejca} obtained his PhD from the Hasso Plattner Institute, University of Potsdam, Germany, in 2019.
    Since 2022, he is an assistant professor at Ecole Polytechnique, Palaiseau, France.
    His research interests are the theoretical analysis of random processes, especially black-box optimization heuristics.
\end{IEEEbiographynophoto}

\end{document}